\documentclass[10pt,twocolumn,letterpaper]{article}

\usepackage[pagenumbers]{style} 

\usepackage{graphicx}
\usepackage{amsmath}
\usepackage{amsthm}
\usepackage{amssymb}
\usepackage{epstopdf}
\usepackage{booktabs}
\usepackage{enumerate}
\usepackage[ruled]{algorithm2e}
\usepackage{appendix}
\usepackage{etoolbox}
\usepackage{xr-hyper}

\usepackage[pagebackref,breaklinks,colorlinks]{hyperref}

\makeatletter
\def\ps@myheadings{%
    \let\@oddfoot\@empty\let\@evenfoot\@empty
    \def\@evenhead{\thepage\hfil\slshape\leftmark}%
    \def\@oddhead{{\slshape\rightmark}\hfil\thepage}%
    \let\@mkboth\@gobbletwo
    \let\sectionmark\@gobble
    \let\subsectionmark\@gobble
    }
  \if@titlepage
  \renewcommand\maketitle{\begin{titlepage}%
  \let\footnotesize\small
  \let\footnoterule\relax
  \let \footnote \thanks
  \null\vfil
  \vskip 60\p@
  \begin{center}%
    {\LARGE \@title \par}%
    \vskip 3em%
    {\large
     \lineskip .75em%
      \begin{tabular}[t]{c}%
        \@author
      \end{tabular}\par}%
      \vskip 1.5em%
    {\large \@date \par}
  \end{center}\par
  \@thanks
  \vfil\null
  \end{titlepage}%
  \setcounter{footnote}{0}%
}
\else
\renewcommand\maketitle{\par
  \begingroup
    \renewcommand\thefootnote{\@fnsymbol\c@footnote}%
    \def\@makefnmark{\rlap{\@textsuperscript{\normalfont\@thefnmark}}}%
    \long\def\@makefntext##1{\parindent 1em\noindent
            \hb@xt@1.8em{%
                \hss\@textsuperscript{\normalfont\@thefnmark}}##1}%
    \if@twocolumn
      \ifnum \col@number=\@ne
        \@maketitle
      \else
        \twocolumn[\@maketitle]%
      \fi
    \else
      \newpage
      \global\@topnum\z@   
      \@maketitle
    \fi
    \thispagestyle{plain}\@thanks
  \endgroup
  \setcounter{footnote}{0}%
}
\makeatother

\usepackage[capitalize]{cleveref}
\crefname{section}{Sec.}{Secs.}
\Crefname{section}{Section}{Sections}
\Crefname{table}{Table}{Tables}
\crefname{table}{Tab.}{Tabs.}

\def\confName{xxxx}
\def\confYear{2022}

\newcommand{\RR}{\mathbb{R}}
\newcommand{\onelip}{1\text{-Lip}}
\newcommand{\spt}{\text{spt}}

\newtheorem{theorem}[]{Theorem}

\newtheorem{definition}[]{Definition}
\newtheorem{proposition}[]{Proposition}
\newtheorem{lemma}[]{Lemma}

\begin{document}

\title{Trust the Critics: Generatorless and Multipurpose WGANs with Initial Convergence Guarantees}

\author{Tristan Milne\textsuperscript{1}\!, {\'E}tienne Bilocq\textsuperscript{1}\!, Adrian Nachman\textsuperscript{1,2}\\
University of Toronto\\
Dept. of Mathematics\textsuperscript{1}\!, The Edward S. Rogers Sr. Dept. of Electrical and Computer Engineering\textsuperscript{2} \\
{\tt\small tmilne \{ebilocq, nachman\}@math.toronto.edu}
}
\maketitle

\begin{abstract}
Inspired by ideas from optimal transport theory we present Trust the Critics (TTC), a new algorithm for generative modelling. This algorithm eliminates the trainable generator from a Wasserstein GAN; instead, it iteratively modifies the source data using gradient descent on a sequence of trained critic networks. This is motivated in part by the misalignment which we observed between the optimal transport directions provided by the gradients of the critic and the directions in which data points actually move when parametrized by a trainable generator. Previous work has arrived at similar ideas from different viewpoints, but our basis in optimal transport theory motivates the choice of an adaptive step size which greatly accelerates convergence compared to a constant step size. Using this step size rule, we prove an initial geometric convergence rate in the case of source distributions with densities. These convergence rates cease to apply only when a non-negligible set of generated data is essentially indistinguishable from real data. Resolving the misalignment issue improves performance, which we demonstrate in experiments that show that given a fixed number of training epochs, TTC produces higher quality images than a comparable WGAN, albeit at increased memory requirements. In addition, TTC provides an iterative formula for the transformed density, which traditional WGANs do not. Finally, TTC can be applied to map any source distribution onto any target; we demonstrate through experiments that TTC can obtain competitive performance in image generation, translation, and denoising without dedicated algorithms.\footnote{Code available at  \url{https://github.com/tmilne5/Trust-the-Critics}.}
\end{abstract}

\section{Introduction}
Wasserstein GANs (WGANs) were introduced in \cite{arjovsky2017wasserstein} as a technique for training generative models using the Wasserstein-1 (or Earth Mover's) distance as a measure of dissimilarity between the real and generated distributions. WGANs consist of two neural networks, a generator and a critic. The generator $G_w$ (with $w$ as its parameters) is applied to a latent probability distribution $\zeta$ to produce the pushforward measure $\mu := (G_w)_\# \zeta$. Given data sampled from a probability distribution $\nu$, the critic $u$ is used to measure the discrepancy between $\mu$ and $\nu$ by computing the Wasserstein-1 distance, given by
\begin{equation}
W_1(\mu, \nu) = \sup_{u} \mathbb{E}_{x \sim \mu}[u(x)] - \mathbb{E}_{y \sim \nu}[u(y)],\label{prob:firstdual}
\end{equation}
where the supremum is computed over all functions $u$ which have a Lipschitz constant of $1$. It was proposed in \cite{gulrajani2017improved} to approximately solve this problem using a gradient penalty on $u$. Given a solution $u_0$ to \eqref{prob:firstdual} (called a Kantorovich potential), the generator is then trained by applying gradient descent to $W_1((G_w)_\# \zeta, \nu)$ over $w$. WGANs have been very successful, becoming one of the standard benchmarks for high quality image synthesis \cite{lucic2017gans, kurach2019large, karras2019style}.

Existing formulations of WGANs do not, however, make full use of the information contained in a Kantorovich potential. Indeed, it follows from $L^1$ optimal transport theory that $-\nabla u_0(G_w(z))$, where it exists, provides the exact direction, if not the distance, in which the generated image $G_w(z)$ should move so as to most rapidly reduce $W_1((G_w)_\# \zeta ,\nu)$. One may ask, then, the extent to which performing gradient descent on $W_1((G_w)_\# \zeta ,\nu)$ over the parameters $w$ actually obtains this optimal direction at each point $G_w(z)$. In Section \ref{sec:misalignment} of this paper we demonstrate experimentally that standard generator architectures fail to obtain this optimal direction.

We were thus motivated to resolve this issue and obtain the ideal direction of flow everywhere by simply modifying the initial distribution by a gradient descent step on $u_0$ with step size $\eta$; that is, to obtain a new distribution, $\tilde{\mu}$, from the formula
\begin{equation}
\tilde{\mu} := (I-\eta\nabla u_0)_\# \mu.\label{eq:singlestep}
\end{equation}
With $\tilde{\mu}$ now in hand, we next train a new critic solving \eqref{prob:firstdual}, with $\tilde{\mu}$ replacing $\mu$, and repeat the process. Given the connection between the gradients of a Kantorovich potential and the optimal transport directions, this process should push $\mu$ towards $\nu$. 

The selection of the step size $\eta$ is non-trivial, however. If $\eta$ is too small the algorithm will require many iterations to converge; since a new critic must be trained after each step, this makes the technique computationally expensive for any distributions $\mu$ and $\nu$ which are not close initially. If $\eta$ is too large, then recalling that $-\nabla u_0(x)$ provides the direction but not the distance of the optimal transport, the gradient step $x \mapsto x-\eta \nabla u_0(x)$ may overshoot the desired target. Ideally, the step size $\eta$ should be derived from the distance that points should move under the optimal transport. In fact, the average of this quantity is equal to $W_1(\mu,\nu)$. Fortuitously, the value of $W_1(\mu,\nu)$ is available to us at no extra computational cost as a by-product of computing $u_0$ via \eqref{prob:firstdual}. 

Hence, in TTC, we select the step size $\eta$ as some fraction of the current estimate of $W_1(\mu,\nu)$. This allows for large step sizes, but avoids undesirable overshooting of the targets. In practice, this enables efficient transformation of $\mu$ to $\nu$ even when they are initially far apart.

In addition, since we have freed $\mu$ from being the output of a generator, TTC can be applied to map any source distribution $\mu$ onto any target $\nu$. Thus,  various tasks in generative modelling, such as image synthesis, translation, and denoising, can be completed using TTC, the only requirement being changing the source and target data. This stands in contrast to the typical approach, where specialized generator architectures or algorithms are designed for each task.

The contributions of this paper are as follows:
\begin{enumerate}
\item We introduce TTC, a novel technique for generative modelling distinguished from existing algorithms by its adaptive step size selection method.
\item We demonstrate quantitatively that standard WGAN architectures fail to move generated data points in optimal directions over the course of training. This misalignment is resolved by TTC.  
\item We provide a formula for the density of the updated measure $\tilde{\mu}$, assuming $\mu$ has a density, and calculate the exact value of $W_1(\tilde{\mu},\nu)$ in terms of the step size $\eta$ under the assumption that the minimal transport distance is larger than this step size.
\item We prove a geometric convergence rate of a sequence of measures generated using TTC. Such a result is applicable until the minimal transport distance is zero, which occurs only when a non-negligible set of synthetic data is essentially indistinguishable from real data.
\item We demonstrate the effectiveness of TTC in practice by showing it generates more realistic looking images than comparable WGANs, as determined by the Fr{\'e}chet Inception Distance (FID) \cite{heusel2017gans}.
\item Finally, we demonstrate the multipurpose applicability of TTC by showing that, in addition to image generation, with just a simple change to the source and target distributions it can be successfully used for image translation as well as denoising tasks.
\end{enumerate}
The remainder of this paper is structured as follows. Related work is reviewed in Section \ref{sec:related}. Section \ref{sec:misalignment} describes experiments which show a large degree of misalignment between the optimal transport directions at generated data points and directions of movement actually obtained when training a generator. In Section \ref{sec:backgroundonOT}, we provide some background on $L^1$ optimal transport theory which is necessary to describe our algorithm. The TTC algorithm is defined in Section \ref{sec:algodef}, and we prove convergence results in Section \ref{sec:theoalgoandconvergence}; detailed proofs of all mathematical results are included in Appendix \ref{app:proofs}. Section \ref{sec:ttcvswgangp} describes an experiment which shows the improved performance of TTC against comparable WGANs, and Section \ref{sec:translation} shows the performance of TTC on the task of translating photographs into paintings in the style of Monet. Section \ref{sec:denoising} points out a connection between TTC and the adversarial regularization technique of \cite{lunz2018adversarial}, and demonstrates that TTC has improved performance versus this technique when applied to an image denoising task. We conclude in Section \ref{sec:discussion and Conclusion} with a discussion of TTC's limitations and potential societal impact.

\section{Related work}
\label{sec:related}
The idea to perform generation without an explicitly parametrized generator, but instead using the gradients of trained critics or discriminators is not new. It was applied in \cite{johnson2018composite} and \cite{lazarow2017introspective} to discriminators trained with the standard GAN objective \cite{goodfellow2014generative}, and in \cite{lee2018wasserstein} and  \cite{nitanda2018gradient} to critics trained with the gradient penalty method from WGANs \cite{gulrajani2017improved}. Since we also train our critics using the gradient penalty method, we will compare our work to \cite{lee2018wasserstein} and \cite{nitanda2018gradient}. In \cite{lee2018wasserstein} generation of new samples is performed by applying Adam \cite{kingma2014adam} with added noise to learned critics. Iterations are halted based on a heuristic stopping condition depending on a random number selected between the minimum and maximum value of the critic on the target samples. This method is not directly connected to the optimal transport problem and consequently has no convergence guarantees. In contrast, our technique applies a simple gradient descent step which is motivated by the concept of transport rays from $L^1$ optimal transport theory; this connection allows us to provide exact initial convergence rates for TTC. Moreover, the size of our step is proportional to the current estimate of the Wasserstein-1 distance, whereas in \cite{lee2018wasserstein} it is a small constant; at the start of training our step size can be quite large, which leads to fast convergence at no cost to the quality of generated samples. We also show superior performance of TTC as compared to WGANs for a fixed generator and critic architecture on three standard datasets, which was not achieved in \cite{lee2018wasserstein}. On the other hand, \cite{lee2018wasserstein} is able to obtain reasonably good performance using only four critic networks, which is more efficient than our technique.

The work in \cite{nitanda2018gradient} contains novel qualitative insights regarding functional gradient descent, and suggests that explicitly parametrized generators in standard WGANs may be unable to correctly follow the gradients of the critics. We had similar insights, and our experimental results in Section \ref{sec:misalignment} provide quantitative evidence for these claims by showing that the actual directions in which generated data move during training are far from the ideal ones. The authors of \cite{nitanda2018gradient} also use a constant step size for the gradient descent, and do not prove that their technique reduces the Wasserstein-1 distance. The  application of their algorithm is to fine-tune pre-trained WGAN generators, possibly because their choice of constant step size leads to long computation times when the source and target are not already close. While the fine-tuning of pre-trained generators is a valuable idea, our adaptive step size is large when the Wasserstein-1 distance between source and target distributions is large, and so we do not have these issues. 

Consequently, our generation scheme is applicable to any source distribution (\ie not necessarily one produced by a pre-trained generator). This allows us to apply our technique not only to image generation, but to image translation and denoising as well; similar flexibility was demonstrated for a different technique in \cite{lazarow2017introspective}. Because of this flexibility, we speculate that our technique is applicable to more than just imaging problems. 

In Section \ref{sec:denoising} we point out a connection between our work and an adversarial regularization technique for inverse problems in \cite{lunz2018adversarial}. In particular, we show that under certain conditions the implicit reconstruction technique used in \cite{lunz2018adversarial} is equivalent to a single step of TTC. While our paper was in development we learned of the recent paper \cite{mukherjee2021end}, which contains, in a similar fashion to our work, iterations of the implicit technique in \cite{lunz2018adversarial}. Our work complements that of \cite{mukherjee2021end}, since we study distinct generative modelling problems and provide an explicit reconstruction technique (as opposed to an implicit one specified by a variational problem) coupled with a proof of initial convergence obtained from $L^1$ optimal transport. Furthermore, \cite{mukherjee2021end} is devoted to a more general class of inverse problems than our denoising example, and we see an extension of our theory to their setting to be an interesting avenue for future work. 

There has been a profusion of recent papers applying concepts from optimal transport to generative modelling \cite{dai2020sliced, finlay2020train, gao2021deep, gao2019deep, mroueh2019sobolev, onken2020ot, sanjabi2018convergence}, some of which have convergence results. These publications tend to focus on optimal transport settings distinct from ours which are typically more amenable to analysis, such as the case where the cost function is the square Euclidean distance \cite{finlay2020train, gao2021deep, onken2020ot}, or where the cost function is regularized by entropy \cite{sanjabi2018convergence}. To our knowledge, our initial convergence results are unique in the $L^1$ optimal transport setting relevant to WGANs.

Finally, let us point out that the quality of our generated images, while better than comparable WGANs, is not comparable to the state of the art (\eg \cite{karras2020training}). However, the intent of the present paper is to focus on fundamental issues arising in adversarial approaches to generative modelling. Given the observed improved performance versus comparable WGANs, it will be interesting to apply TTC with state of the art discriminator architectures, but this is outside the scope of the current paper.

\section{Misalignment of data movement}
\label{sec:misalignment}
Given a Kantorovich potential $u_0$ for a generated distribution $\mu = (G_w)_\#\zeta$ and a target distribution $\nu$, the optimal direction of movement for a generated sample $G_w(z)$ is provided by the gradient $-\nabla u_0(G_w(z))$. When training a WGAN generator, however, the movement of generated samples is ruled by the gradient descent algorithm used to minimize $W_1((G_w)_\# \zeta, \nu)$ in terms of $w$. A standard gradient descent step will update $w$ in the direction of
\begin{align}
-\nabla_w W_1((&G_{w})_\# \zeta, \nu) = -\nabla_w \mathbb{E}_{(G_{w})_\# \zeta}[u_0], \\
&=-\mathbb{E}_{\tilde{z} \sim \zeta}\left[D_w G_w^T(\tilde{z}) \nabla u_0(G_w(\tilde{z}))\right].
\end{align}
When the expectation is calculated by averaging over a random mini-batch $\{z_i\}_{i=1}^M$ (\ie when using stochastic gradient descent), this leads to movement of a generated sample $G_w(z)$ parallel to
\begin{equation}
-D_wG_w(z) \left(\frac{1}{M} \sum_{i=1}^MD_w G_w^T(z_i) \nabla u_0(G_w(z_i))\right). \label{eq:actualupdatedirection}
\end{equation}
In this section we present our results on the misalignment between this expression and $-\nabla u_0(G_w(z))$.

We computed the cosines of the angles between vectors of the form \eqref{eq:actualupdatedirection} \ and $-\nabla u_0 (G_w(z))$ at randomly generated samples $G_w(z)$ and at various stages throughout the training of a WGAN generator, using mini-batches $\{z_i\}_{i=1}^M$ drawn separately from $z$. The critic network used to approximate $u_0$ was trained using the one sided gradient penalty from \cite{gulrajani2017improved}. The experiment was run on the MNIST dataset using both the InfoGAN architecture (\hspace{1sp}\cite{chen2016infogan}, which we selected since it was also used in the large-scale study \cite{lucic2017gans}) and the DCGAN architecture \cite{Radford2016UnsupervisedRL} for comparison. We used the batch size $M=128$. The results are summarized in Table \ref{table:ndp1}, and histograms of the cosine values obtained are provided in Appendix \ref{sec:misalignment_appendix}. The results clearly demonstrate that the movement of generated samples during WGAN training is not well aligned with the optimal directions provided by the Kantorovich potential. This discrepancy tends to increase as training goes on, to the point where negative cosine values (denoting undesirable movement in a direction along which $u_0$ increases) are obtained for a non-negligible portion of the generated samples, which may help explain why a generator can stop improving before the generated distribution is sufficiently close to the target. Table \ref{table:ndp1} also presents the performance of the generators as indicated by the FID. The InfoGAN generator obtained higher cosine values and lower FIDs than the DCGAN generator, suggesting that better alignment throughout training correlates with better performance. In Appendix \ref{sec:misalignment_appendix}, we also include a similar experiment describing the misalignment obtained when replacing \eqref{eq:actualupdatedirection} by updates computed with Adam.

\begin{table}
\centering
\begin{tabular}{ccc}
\toprule
Stage in training & InfoGAN & DCGAN  \\
\midrule
Early & $0.488\pm0.071$ & $0.441\pm0.075$ \\
 & \small{($410.4\pm9.0$)} & \small{($414.3\pm26.8$)} \\
\midrule
Mid & $0.150\pm0.080$  & $0.086\pm0.081$\\
 & \small{($22.9\pm 0.5$)} & \small{($40.8\pm7.1$)} \\
\midrule
Late & $0.100\pm0.070$ & $0.085\pm0.081$ \\
 & \small{($21.8\pm 0.4$)} & \small{($31.1\pm1.8$)} \\
\bottomrule
\end{tabular}
\caption{Misalignment cosine measurements. All results are reported in the form mean $\pm$ standard deviation. The top lines show the statistics of the cosines of the angles between \eqref{eq:actualupdatedirection} and $-\nabla u_0(G_w(z))$, computed over the training of five separate initializations of each model, each with 256 randomly sampled noise inputs $z$. The statistics appearing in parentheses are for FID values computed over the five training runs. The values were computed at three different stages in training, with the late stage chosen such that the generators were close to the best performance obtained, but still improving slightly.}
\label{table:ndp1}
\end{table}

\section{The TTC algorithm}
\subsection{Background on $L^1$ optimal transport theory}
\label{sec:backgroundonOT}
Here we will summarize the necessary background from $L^1$ optimal transport; for more details, see, for instance,  \cite{santambrogio2015optimal}. Let $\Omega \subset \RR^d$ be a compact convex set representing the space of images with $c$ colour channels and $h \times w$ pixels per channel (i.e. $d = c\times h \times w$). Let $\mu, \nu$, the source and target, respectively, be in $\mathcal{P}(\Omega)$, the space of probability distributions on $\Omega$. The $L^1$ optimal transport problem is the problem of finding a map $T: \Omega \rightarrow \Omega$ mapping $\mu$ onto $\nu$ while minimizing the average Euclidean travel distance. Formally, it is
\begin{equation}
\min_{T_\# \mu = \nu} \int_\Omega |x-T(x)| d\mu, \label{prob:L1OTproblem}
\end{equation} 
where the constraint $T_\# \mu = \nu$ specifies that the pushforward measure $T_\# \mu$, defined by the equality
\begin{equation}
T_\# \mu (E) = \mu(T^{-1}(E)),
\end{equation}
agrees with $\nu$. A solution $T_0$ to \eqref{prob:L1OTproblem} is called an optimal transport map for the pair $(\mu,\nu)$, and the existence of such a map can be guaranteed by the assumption that $\mu$ has a density with respect to Lebesgue measure on $\RR^d$, an assumption we will denote with the notation $\mu \ll \mathcal{L}_d$ (see, for example, Theorem 3.18 in \cite{santambrogio2015optimal}).

In general the existence of an optimal map $T_0$ is difficult to prove. In the Kantorovich relaxation of \eqref{prob:L1OTproblem}, one instead minimizes over probability distributions $\gamma \in \mathcal{P}(\Omega \times \Omega)$ with marginals given by $\mu$ and $\nu$. Formally, we write this as
\begin{align}
\min_{\gamma \in \Pi(\mu,\nu)}& \int_{\Omega\times \Omega} |x-y| d\gamma, \label{prob:Kantorovichrelaxation} \\ 
\Pi(\mu,\nu) := \{ \gamma \in \mathcal{P}(\Omega \times \Omega) &\mid (\pi_x)_\# \gamma = \mu, (\pi_y)_\# \gamma = \nu\}. 
\end{align}
Here $\pi_x(x,y) =x$, and $\pi_y(x,y) =y$ are the projection maps, and $\Pi(\mu,\nu)$ is called the set of admissible plans. The value of the minimization problem in \eqref{prob:Kantorovichrelaxation} is written $W_1(\mu,\nu)$ and is called the Wasserstein-1 distance between $\mu$ and $\nu$.

Under certain conditions (see Section 1.5 of \cite{santambrogio2015optimal}) the values of problems \eqref{prob:L1OTproblem} and \eqref{prob:Kantorovichrelaxation} agree. The existence of an optimal map is typically shown starting from Kantorovich Duality, which asserts that the optimal values in \eqref{prob:firstdual} and \eqref{prob:Kantorovichrelaxation} are the same. The relationship between a Kantorovich potential $u_0$ and an optimal map $T_0$ is well known (\eg \cite{evans1999differential, gulrajani2017improved}) but we include a statement of it for completeness.
\begin{lemma}
Let $\mu \ll \mathcal{L}_d$, $u_0$ be a Kantorovich potential and $T_0$ be an optimal map for the pair $(\mu,\nu)$. Then $\mu$-almost everywhere,\label{lem:graduistransportdirection}
\begin{equation}
x \neq T_0(x) \Rightarrow -\nabla u_0(x) = \frac{T_0(x)-x}{|T_0(x)-x|}.\label{eq:transportdirection}
\end{equation}
\end{lemma}
Equation \eqref{eq:transportdirection} implies that  given a step size $\eta_0$, if we were to apply the map
\begin{equation}
f_0(x) = x - \eta_0 \nabla u_0(x),\label{eq:effdeffinformal}
\end{equation}
then as long as $\eta_0$ is less than the minimal transport length we can be assured that under $f_0$ all points $x$ move towards their targets $T_0(x)$ without overshooting them, and thus we can expect that $W_1((f_0)_\# \mu,\nu)$ will be less than $W_1(\mu,\nu)$. This observation forms the basis for our algorithm, which we will describe in the next section.
\subsection{The algorithm}
\label{sec:algodef}
We now define the TTC algorithm. Set $\mu_0:= \mu$, and define, for $n =1 , 2 \ldots$
\begin{equation}
\mu_{n} = (f_{n-1})_\# \mu_{n-1}\label{eq:mundef},
\end{equation}
where $f_{n-1}$ is as in \eqref{eq:effdeffinformal}, but with $u_0$ replaced by a Kantorovich potential $u_{n-1}$ for the pair $(\mu_{n-1}, \nu)$. To compute such a Kantorovich potential we train standard critic neural networks from the literature using the one sided gradient penalty from \cite{gulrajani2017improved}, which was found to provide more stable training  \cite{petzka2018regularization}. In order to use this technique it is necessary to be able to sample both $\mu_n$ and $\nu$.  For $\mu_n$, we draw an initial point $x_0$ from $\mu_0$, and apply gradient descent maps from the sequence of pre-trained critics. Precisely, a sample $x \sim \mu_n$ is obtained via the formula
\begin{equation}
x = (I-\eta_{n-1}\nabla u_{n-1}) \circ \ldots \circ (I-\eta_0 \nabla u_0)(x_0).\label{eq:generator_formula}
\end{equation}
Note that in light of \eqref{eq:generator_formula}, we can view TTC as building up a generator layer by layer using the gradients of the critics, as in \cite{nitanda2018gradient}.

We choose the step size via
\begin{equation}
\eta_n = \theta W_1(\mu_n,\nu),\label{eq:stepsizedef}
\end{equation}
where $\theta \in (0,1)$ is a hyperparameter. We note again that the value of $W_1(\mu_n,\nu)$ is available as a by-product of computing a Kantorovich potential, so this choice of adaptive step size requires no extra computation in practice. More precisely, we use the negative of the minimal value of the functional from WGAN-GP \cite{gulrajani2017improved}, that is
\begin{equation}
W_1(\mu_n,\nu) \approx \frac{1}{M}\sum_{j=1}^M u_n(x_j) - u_n(y_j) - \lambda (|\nabla u_n(\tilde{x}_j)| - 1)_+^2, \label{eq:w1approx}
\end{equation}
where the $x_j$ and $y_j$ are samples from $\mu_n$ and $\nu$ respectively, $M$ is the mini-batch size, $\tilde{x}_j$ is a random convex combination of $x_j$ and $y_j$ as in \cite{gulrajani2017improved} and $(z)_+ = \max(0, z)$. When training WGANs, researchers will typically use $\lambda = 10$ for the gradient penalty coefficient \cite{gulrajani2017improved, lunz2018adversarial, mukherjee2021end}, however we use a value of $\lambda = 1000$. In practice we found that this value of $\lambda$ stabilizes the estimates of $W_1(\mu_n,\nu)$ and the training of TTC; using smaller values of $\lambda$ leads to inflated estimates of $W_1(\mu_n,\nu)$, leading to overly large step sizes and unstable training. This is confirmed by the recent analysis in \cite{milne2021wasserstein}, which shows that at best the value in \eqref{eq:w1approx}, in expectation, converges to $W_1(\mu_n,\nu)$ like $O(\lambda^{-1})$. 

Depending on the mini-batch size $M$ the value of \eqref{eq:w1approx} can vary considerably across mini-batches, so we compute an average over the last $100$ mini-batches used to train the critic $u_n$. The method for training TTC is summarized in Algorithm \ref{alg:TTC}. 
\begin{algorithm}
\SetAlgoLined
\KwData{Samples from source $\mu$ and target $\nu$, untrained critics $(u_n)_{n=0}^{N-1}$ with parameters $(w_n)_{n=0}^{N-1}$, gradient penalty coefficient $\lambda$, number of critic iterations $C$, batch size $M$, Adam parameters $(\epsilon_c, \beta_1,\beta_2)$.}
\KwResult{A distribution $\mu_n$ which can be sampled from $\mu$, $(u_n)_{n=0}^{N-1}, (\eta_n)_{n=0}^{N-1}$ via \eqref{eq:generator_formula}.}
\For{$n \in \{0, \ldots, N-1\}$}{
 \For{$i\in \{1, \ldots, C\}$}{
  \For{$j \in \{1, \ldots, M\}$}{
  	Sample $x_j \sim \mu_0$, $y_j \sim \nu$, $t_j \sim U([0,1])$\;
  	\For{$k \in \{0, \ldots, n-1\}$}{
  		$x_j \leftarrow x_j- \eta_k \nabla u_k(x_j)$\;}
  		$\tilde{x}_j \leftarrow (1-t_j)x_j + t_jy_j$\;
  		$L_{ij} \leftarrow u_n(y_j) - u_n(x_j)$ \\
  		$\hspace{0.4 in}+ \lambda (|\nabla u_n(\tilde{x}_j)|-1)_+^2$\;
  		  		}
  		$L_i \leftarrow \frac{1}{M}\displaystyle{\sum_{j=1}^M} L_{ij}$ \;
  		$w_n \leftarrow \text{Adam}(L_i, \epsilon_c, \beta_1,\beta_2)$ \;}
  $\eta_n \leftarrow \frac{-\theta}{100}\displaystyle{\sum_{i=C-99}^C}L_i$\;
  		}
\caption{TTC Training}\label{alg:TTC}
\end{algorithm}

\subsection{Convergence results for TTC}
\label{sec:theoalgoandconvergence}
In this section we will provide guarantees that up to some iterate, $W_1(\mu_n, \nu)$ decreases with geometric rate determined by $\theta$. Detailed proofs for all mathematical results can be found in Appendix \ref{app:proofs}.

Let $u_n$ be a Kantorovich potential for the pair $(\mu_n, \nu)$. Since $u_n$ is Lipschitz continuous, $\nabla u_n$ is \textit{a priori} only defined almost everywhere. To provide a convergence analysis, it is convenient to extend $f_n$ to the negligible set where $u_n$ is not differentiable via the formula 
\begin{equation}
f_n(x) :=\begin{cases}
x- \eta_n \nabla u_n(x) &\quad {\text{ if }\nabla u_n(x) \text{ exists}},\\
x &\quad \text{ else.}
\end{cases}	  \label{eq:effdeff}
\end{equation}
We emphasize that this extension does not affect the definition of the measures $\mu_n$ as long as they have densities, which we will prove up to a certain iterate. In order to state our main results we first need to define the minimal transportation length in an appropriate way.
\begin{definition}
For $\mu, \nu \in \mathcal{P}(\Omega)$, $\mu \ll \mathcal{L}_d$, let $T_0$ be an optimal transport map for the pair $(\mu,\nu)$. Set
\begin{align}
\ell_0(T_0) &:= \text{essinf}_{\mu} |I - T_0|,\\
&= \sup\{ \ell \in \RR  \mid \mu\left(\{ x \in \Omega \mid |x-T_0(x)| <\ell \}\right) = 0\}. 
\end{align}
\end{definition}
Our first theorem gives precise information on the effects of one iteration of TTC.
\begin{theorem}
Suppose that $\mu \ll \mathcal{L}_d$ and let $u_0$ be a Kantorovich potential for the pair $(\mu,\nu)$. Set $\tilde{\mu}= (f_0)_\# \mu$ where $f_0$ is defined in \eqref{eq:effdeff}. Let $T_0$ be an optimal transport map for the pair $(\mu,\nu)$. If 
\begin{equation}
\eta_0 < \ell_0(T_0),\label{eq:stepsizelessthanminimallength}
\end{equation}
then
\begin{enumerate}[i.]
\item $u_0$ is a Kantorovich potential for the pair $(\tilde{\mu},\nu)$,\label{item:u0iskantpot}
\item $W_1(\tilde{\mu},\nu) = W_1(\mu,\nu) - \eta_0$,\label{item:exactw1value}
\item With $g_0: \Omega \rightarrow \RR^d$ defined by\label{item:optimalmap}
\begin{equation}
g_0(x) = \begin{cases}
x + \eta_0 \nabla u_0(x) &\quad {\nabla u_0(x) \text{ exists}},\\
x &\quad \text{else,}
\end{cases} 
\end{equation}
then $T_{1}:= T_0 \circ g_0$ is an optimal map for the pair $(\tilde{\mu},\nu)$ and
\begin{equation}
\ell_0(T_1) \geq \ell_0(T_0) - \eta.
\end{equation}
\item $\tilde{\mu}$ has density $\tilde{\rho}$ satisfying, for $\mu$ almost all $x$,
\begin{equation}
    \tilde{\rho}(f_0(x)) = \rho(x) |D g_0(f_0(x))|,\label{eq:densityformula}
\end{equation}
where $\rho(x)$ is the density for $\mu$. \label{item:densityformula}
\end{enumerate}\label{prop:propertiesofnextmu}
\end{theorem}

We now iterate Theorem \ref{prop:propertiesofnextmu} (specifically part \ref{item:exactw1value}) to obtain our convergence results. In view of \eqref{eq:stepsizelessthanminimallength}, this requires that $\eta_n$ is less than the minimal transportation length for the pair $(\mu_n, \nu)$. Recalling the definition of $\eta_n$ (see \eqref{eq:stepsizedef}) and that $W_1(\mu_n,\nu)$ is the mean transportation length according to $\mu_n$ (see \eqref{prob:L1OTproblem}), we see that $\eta_n$ will satisfy this requirement provided the transport distances are  clustered around the mean and $\theta$ is small enough.

More precisely, the next theorem shows that up to an iterate which is determined by the ratio $\ell_0(T_0)/W_1(\mu,\nu)$, which measures the clustering of the transport distances around the mean, the sequence $(\mu_n)_{n=0}^\infty$ converges to $\nu$ at a geometric rate.
\begin{theorem}
Let $T_0$ be an optimal transport map for the pair $(\mu,\nu)$. The measure $\mu_n$ defined iteratively by \eqref{eq:mundef} satisfies
\begin{equation}
W_1(\mu_n,\nu) = (1-\theta)^nW_1(\mu,\nu)\label{eq:W1estimate}
\end{equation}
for all $n \in \{ 0, \ldots, N(\theta)\}$, where
\begin{equation}
N(\theta) = \lceil \log_{1-\theta}\left(1-\ell_0(T_0)/W_1(\mu_0,\nu)\right) \rceil -1,
\end{equation}
where $\lceil \cdot \rceil$ denotes the ceiling function. In particular, if $\ell(T_0) < W_1(\mu,\nu)$, then
\begin{equation}
W_1(\mu_{N(\theta)}, \nu) \leq \frac{W_1(\mu,\nu) - \ell_0(T_0)}{1-\theta}.
\end{equation}
\label{prop:approximateconvergence}
\end{theorem}
Theorem \ref{prop:approximateconvergence} gives us precise knowledge of $W_1(\mu_n,\nu)$ up until $n = N(\theta)$. At this point the convergence rate ceases to apply because the step size $\eta_{N(\theta)}$ will exceed the current minimal transport length. In principle, we may then reduce $\theta$ until the step is small enough and resume iterations. In practice, however, this is difficult because the minimal transport length is unknown. Further, we found that a constant value of $\theta$ was sufficient to produce high quality results. Note that if no positive value of $\theta$ will make the current step size less than the minimal transport length, then this length must be zero, so a non-negligible set of generated data is essentially indistinguishable from real data. 

Let us also observe that according to Theorem \ref{prop:propertiesofnextmu}, the initial $u_0$ will be a Kantorovich potential for the pair $(\mu_n, \nu)$ until $\eta_n$ is larger than the minimal transport distance for that pair, and hence could be re-used as the critic. In the TTC algorithm, however, we train a new critic at each step. This is for two reasons. First, since the minimal transport distance is unknown to us, it would be difficult to know when we should begin re-training $u_0$. Second, since $u_0$ is parametrized by a neural network, it may not be precisely equal to the ideal Kantorovich potential. Motivated by Theorem \ref{prop:propertiesofnextmu}, however, we initialize the parameters of $u_n$ at the parameters of the preceding critic $u_{n-1}$, which we found leads to improved performance.

\section{Experiments}
We apply TTC to image generation, translation, and denoising. Regarding image generation, we test TTC against WGAN-GP for a fixed generator and discriminator on three standard datasets and show that TTC produces more realistic images, as measured by lower FID. We then apply TTC to the translation of photographs to Monet paintings using the datasets in \cite{zhu2017unpaired}. We finish with an application of TTC to image denoising inspired by \cite{lunz2018adversarial}.
\subsection{Comparison to WGAN-GP}
\label{sec:ttcvswgangp}
As a first experiment we compare TTC to WGAN-GP on the tasks of generating the MNIST \cite{lecun1998gradient}, FashionMNIST (abbreviated here as F-MNIST) \cite{xiao2017fashion}, and CIFAR-10 \cite{krizhevsky2009learning} datasets. We select the InfoGAN architecture \cite{chen2016infogan} for our generator and discriminator. We train each generative technique for a fixed number of minibatches of training data (100,000 for MNIST/F-MNIST, and 200,000 for CIFAR-10). In order to have a fair starting point, when training TTC we initialize the source distribution $\mu_0$ to be the output of the corresponding untrained generator from WGAN-GP; thus, the initial FIDs for both techniques are equal when given the same initialization. We also run a similar experiment for TTC that uses the standard normal distribution on $\RR^d$ as the source $\mu_0$. The hyperparameters for TTC and WGAN-GP for each experiment are summarized in Table \ref{table:TTCvsWGANGPsetup}. Where possible, the same hyperparameters are used for both techniques. Important exceptions to this are $\theta$ and $\epsilon_g$ (the generator learning rate) which are specific to TTC and WGAN-GP, respectively, and strongly impact performance. Both of these hyperparameters are optimized via a grid search for each dataset, and the best performing value is used for all experiments. For all datasets and algorithms we measure performance by calculating the FID using a PyTorch implementation \cite{Seitzer2020FID} every $n$ minibatches ($n=10{\small,}000$ for MNIST/F-MNIST, and $n = 20{\small,}000$ for CIFAR-10). FIDs are computed by generating a set of $10{\small,}000$ images and comparing these to the standard test set from each dataset. Both techniques are trained using five random initializations, and the average values of the best FID over each training run are recorded in Table \ref{table:TTCvsWGANGP}. Some examples of generated images are given in Appendix \ref{sec:generative_appendix}.

\begin{table}
\centering
\begin{tabular}{lcccc}
\toprule
\multicolumn{1}{l}{Dataset} &
\multicolumn{3}{c}{TTC}    &
\multicolumn{1}{c}{WGAN-GP}    \\ 
\cmidrule(lr){2-4}
\cmidrule(lr){5-5}
 & $N$ & $C$ & $\theta$ & $\epsilon_g$  \\
\midrule
MNIST & $20$ &  $5000$ & $0.9$ & $0.001$ \\
\midrule
F-MNIST & $20$ & $5000$ & $0.9$ & $0.001$ \\
\midrule
CIFAR-10 & $40$ & $5000$ & $0.9$ & $0.0005$ \\
\bottomrule
\end{tabular}
\caption{Experimental setup for comparing WGAN-GP to TTC. The symbols $\epsilon_c, \epsilon_g$ represent the learning rates of the critic and generator, respectively.  The values of the parameters in Algorithm \ref{alg:TTC} were taken to be $\lambda = 1000, M = 50, \epsilon_c = 0.0001, \beta_1 = 0.5, \beta_2 = 0.999$  for all experiments and both techniques. Note that twice as many critics are used for CIFAR-10 experiments on TTC since twice as many minibatches of training are available.}\label{table:TTCvsWGANGPsetup}
\end{table}

\begin{table}
\centering
\begin{tabular}{lccc}
\toprule
Dataset & WGAN-GP & TTC 1 & TTC 2  \\
\midrule
MNIST & $20.9 \pm 0.6$  & $19.0 \pm 1.5$ & $\mathbf{18.0} \pm 0.3$ \\
\midrule
F-MNIST & $26.9 \pm 1.6$ & $22.2 \pm 0.7$ & $\mathbf{22.1} \pm 0.5$\\
\midrule
CIFAR-10 & $29.2 \pm 0.7$ & $28.7 \pm 0.6$ & $\mathbf{27.3} \pm 0.6$\\
\bottomrule
\end{tabular}
\caption{Best FIDs obtained over the course of training for WGAN-GP and TTC. TTC 1 uses an untrained generator to create the source $\mu_0$, while TTC 2 starts from a standard normal distribution on $\RR^d$. The FIDs are reported as mean $\pm$ standard deviation, with statistics computed over five training runs. TTC reliably produces better FID values than WGAN-GP for all datasets; the improvement is magnified when using the standard normal distribution on $\RR^d$ as $\mu_0$.}\label{table:TTCvsWGANGP}
\end{table}

\subsection{Image translation}
\label{sec:translation}
By simply changing the source and target distributions we can use TTC for image translation. We demonstrate this here using datasets from \cite{zhu2017unpaired}. Specifically, we used their Photograph dataset as $\mu_0$, rescaled to have $h=w=128$, and used their Monet dataset as $\nu$, with each target image taken as a random crop of a Monet painting with $h=w=128$. Since the InfoGAN discriminator we used in Section \ref{sec:ttcvswgangp} is not appropriate for images of this size, here we choose the SNDCGAN discriminator from \cite{kurach2019large} (see also \cite{miyato2018spectral}).

Besides making $N = 30$ and $M = 16$ the hyperparameters for this application are exactly the same as in Table \ref{table:TTCvsWGANGPsetup}. An example photograph $x_0$ and the result of applying \eqref{eq:generator_formula} to it with $n = 30$ is given in \Cref{fig:translationexample}; additional examples can be found in Appendix \ref{app:translation}.

\begin{figure}
    \centering
    \includegraphics[width = 0.45 \textwidth, viewport=50 85 415 260, clip = True]{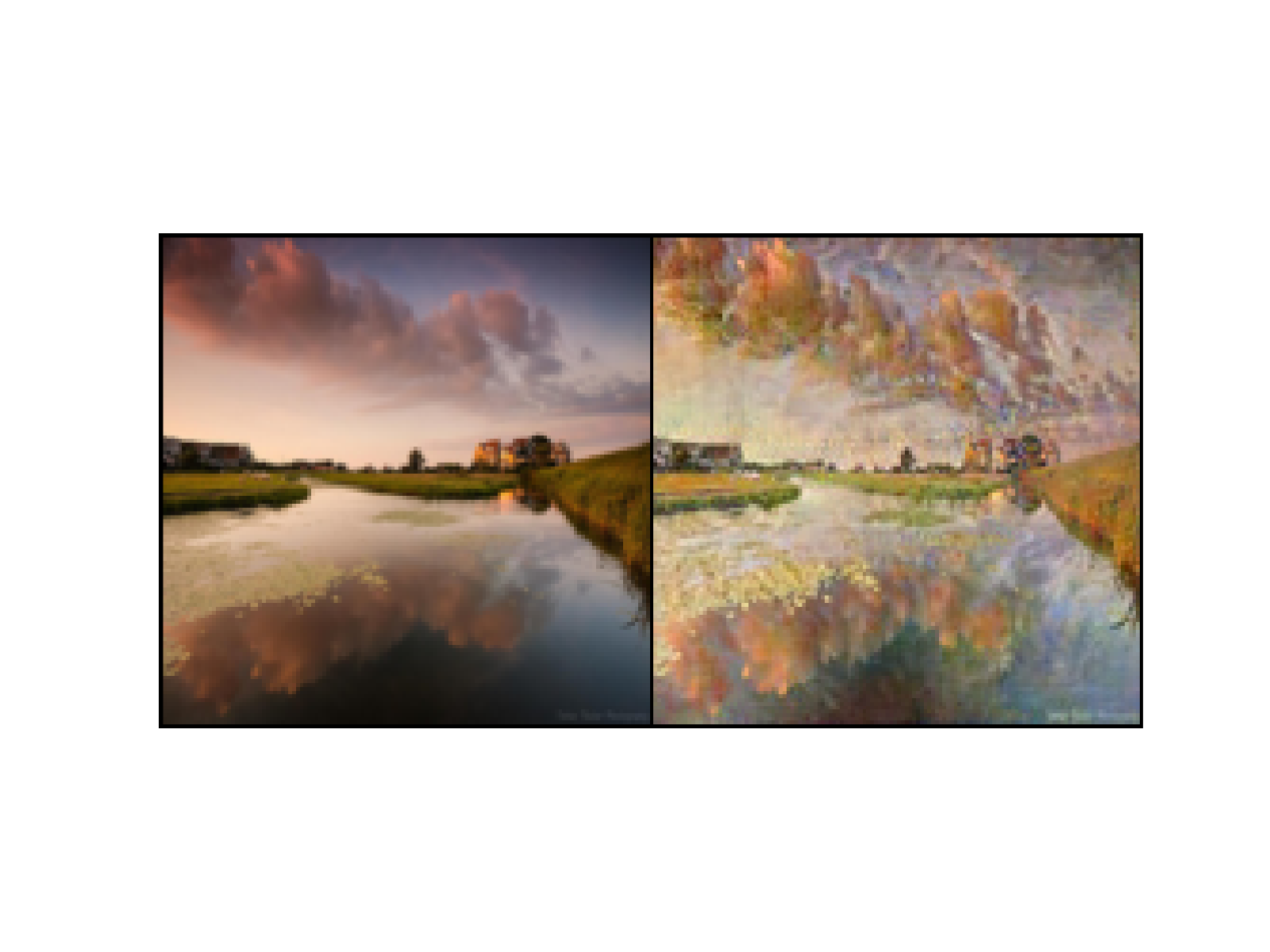}
    \caption{An example of TTC applied to translating landscape photos into Monet paintings.}\label{fig:translationexample}
\end{figure}

\subsection{Image denoising}
\label{sec:denoising}
To further demonstrate the multipurpose flexibility of TTC, we apply it to restore images that have been corrupted with Gaussian noise. Specifically, we follow the experimental framework of \cite{lunz2018adversarial}, where $\nu$ consists of random crops of the BSDS500 dataset \cite{arbelaez2010contour}, and $\mu_0$ is $\nu$ corrupted by adding i.i.d. Gaussian noise with standard deviation $\sigma$. In this setting TTC bears an interesting relation to \cite{lunz2018adversarial}. In that paper a critic $u_0$ is obtained using the method from \cite{gulrajani2017improved} for the pair $(\mu_0, \nu)$, which is then used as a learned regularizer in an inverse problem. This is applied to image restoration in the following way; given a noisy observation $x_0$, a denoised version is obtained by solving the minimization problem
\begin{equation}
\min_{x \in \Omega} \frac{1}{2}|x-x_0|^2 + \eta u_0(x),\label{prob:advreg}
\end{equation} 
where the parameter $\eta$ is estimated from the noise statistics. Incidentally, this requires the noise model to be known \textit{a priori}, as in \cite{moran2020noisier2noise}; as we will see, TTC requires no such model since the corresponding step size parameter is estimated with $\theta W_1(\mu,\nu)$. The next proposition shows that provided $\eta$ is small enough the solution to \eqref{prob:advreg} is equivalent to the solution obtained from a single step of TTC with step size $\eta$. As such, in this context TTC can be thought of as an iterated form of the technique in \cite{lunz2018adversarial}, where the critic is updated after each reconstruction step.
\begin{proposition}
\label{prop:advregis1ttc}
Let $T_0$ be an optimal transport map for the pair $(\mu,\nu)$. If $\eta < \ell_0(T_0)$, then for $\mu$-almost all $x_0$ there is a unique solution to \eqref{prob:advreg} given by
\begin{equation}
x_1 = x_0 - \eta \nabla u_0(x_0).
\end{equation}
\end{proposition}
For a fair comparison of our results against the denoising technique in \cite{lunz2018adversarial} we use a discriminator from \cite{lunz2018adversarial} for this application. For TTC's hyperparameters, we use $N = 20$, $M = 16$ and $\theta = 0.7$, with all other hyperparameters as given in Table \ref{table:TTCvsWGANGPsetup}; interestingly, we found in this case that the performance is quite insensitive to variations in $\theta$. Table \ref{table:denoisingresults} shows mean PSNR values over a set of 128x128 crops from a held out test set from BSDS500 with 200 images. Figure \ref{fig:deer} shows one denoised example; additional examples are in Appendix \ref{app:denoising}. Though our PSNR values are somewhat lower than the state of the art (\eg \cite{moran2020noisier2noise}), we feel that our results are impressive given that we use unpaired data, we have no noise model, and that our technique was not specifically designed for image denoising. 

\begin{table}
\centering
\begin{tabular}{cccc}
\toprule
\multicolumn{1}{l}{} &
\multicolumn{3}{c}{Average PSNR (dB)} \\ 
\cmidrule(lr){2-4}
 $\sigma$ & Noisy Image & Adv. Reg. \cite{lunz2018adversarial} & TTC  \\
\midrule
0.1 & $20.0 \pm 0.03$ &  $27.1 \pm 0.7$ & $\mathbf{29.9} \pm 2.2$ \\
\midrule
0.15 & $16.5 \pm 0.03$ & $24.6 \pm 0.7$ & $\mathbf{27.9} \pm 2.3$ \\
\midrule
0.2 & $14.0 \pm 0.03$ & $22.9 \pm 0.6$ & $\mathbf{26.4} \pm 2.4$ \\
\bottomrule
\end{tabular}
\caption{Results for denoising experiments. PSNR values are reported as mean $\pm$ standard deviation, where the statistics are computed over the test set. In addition to having higher mean performance over the test set, TTC gives an improved PSNR for every image in the test set.}
\label{table:denoisingresults}
\end{table}
\begin{figure}
    \centering
    \includegraphics[width = 0.45 \textwidth, viewport=50 125 420 225, clip = True]{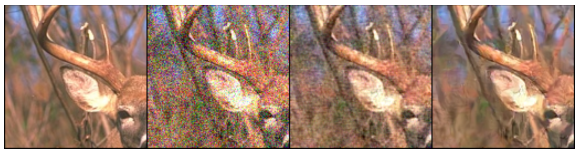}
    \caption{A single restored image from BSDS500 from noise level $\sigma = 0.15$. From left to right: original image, noisy image (PSNR = 16.4), restored image using \cite{lunz2018adversarial} (PSNR = 24.5), restored image using TTC (PSNR = 27.0).}\label{fig:deer}
\end{figure}

\section{Discussion and conclusion}
\label{sec:discussion and Conclusion}
\subsection{Discussion of limitations and societal impact}
In this section we will discuss the limitations and potential societal impact of our contribution. To start with, our theoretical analysis relies on being able to compute a Kantorovich potential $u_0$ for the pair $(\mu,\nu)$. Given recent results in \cite{milne2021wasserstein} showing that the optimization problem for learning critics from \cite{gulrajani2017improved} actually returns a function that solves a congested transport problem, which is distinct from a Kantorovich potential, one may ask to what extent our theoretical analysis applies to the types of critics learned in practice. The distinction between these problems decreases as $\lambda$ gets large, so we suspect that our assumption of being able to compute a Kantorovich potential is reasonable for the $\lambda$ value we used ($\lambda = 1000$), but this is certainly something to consider, especially for the smaller values of $\lambda$ typically used in the literature.

Disadvantages of TTC are its memory and time complexity, which are $O(N)$ and $O(N^2)$, respectively. Nevertheless, standard critic architectures in WGANs are typically simpler than their corresponding generators, so TTC can be faster than traditional WGANs for low $N$. In fact, we were able to generate high quality images in comparable time to WGANs without implementing ideas to reduce the time and memory complexity. 

TTC has outperformed WGANs in our experiments. However, given the results of \cite{lucic2017gans} showing that with enough hyperparameter optimization and random restarts many generative models achieve comparable performance to each other, it would be interesting to test whether the same trend holds for TTC versus WGANs.

In terms of societal impact, since we have demonstrated that TTC improves upon the performance of WGANs, in principle it may be used to generate more convincing deep-fakes, which pose a serious threat to the integrity of information online.  Though we are not experts in the detection of deep-fakes, we can speculate that the understanding of the training process furnished by our theoretical results may help the detection of fake data generated with TTC.

\subsection{Conclusion}
In this paper we have introduced TTC, a new algorithm for generative modelling. This was motivated in part by our observation that the directions in which the training of WGANs moves generated data are misaligned with the ideal directions given by a trained critic. TTC resolves this issue by dispensing with the generator altogether. We then presented a convergence analysis for TTC which applies until the algorithm produces a non-negligible set of high-quality data. As well, we presented experiments successfully applying our algorithm to image generation, translation, and denoising, illustrating the multipurpose nature of TTC.  

Furthermore, we provided an explicit formula for the updated density of $\mu$, which is characteristic of normalizing flows \cite{dinh2014nice, rezende2015variational, chen2018neural, grathwohl2018ffjord}. A trade-off between sample quality and density estimation is often discussed in the literature \cite{grover2017flow, theis2015note}. Since TTC generates high quality samples and has an iterative formula for the density, it may help to bridge this gap.  

\subsection*{Acknowledgements}
This research was supported in part by the NSERC Discovery Grant RGPIN-06329 and a University of Toronto Doctoral Completion Award. Experiments were run on the Graham and Narval clusters of Compute Canada.

{\small
\bibliographystyle{ieee_fullname}%
\bibliography{AnnotatedBibliography} 
}

\clearpage
\begin{appendices}

\section*{Appendix}
\section{Proofs of the mathematical results}
\label{app:proofs}
In order to prove \Cref{lem:graduistransportdirection}, \Cref{prop:propertiesofnextmu}, \Cref{prop:approximateconvergence} and \Cref{prop:advregis1ttc} we will need some terminology and results from the theory of $L^1$ optimal transport. Central to the story are the concept of transport rays, a term coined in \cite{evans1999differential}, which refers to segments over which the Lipschitz inequality of a $1$-Lipschitz function is saturated.
\begin{definition}
Let $u \in \onelip(\Omega)$, the set of Lipschitz functions on $\Omega$ with constant $1$. For $x, y \in \Omega$ the segment $[x,y] := \{ (1-t)x + ty \mid t \in [0,1]\}$ is called a transport ray of $u$ if
\begin{enumerate}
\item $x \neq y$,
\item $u(x) - u(y) = |x-y|$,
\item $[x,y]$ is not properly contained in any other segment $[z,w]$ satisfying properties 1 and 2.
\end{enumerate}
The open segment $]x,y[ := \{ (1-t)x + ty \mid t \in (0,1)\}$ is called the interior of the transport ray, and $x$ and $y$ are called its upper and lower endpoints respectively.
\end{definition}
\begin{lemma}[Essentially Lemmas 3.5 and 3.6 from \cite{santambrogio2015optimal}]
If $[x,y]$ is a transport ray of $u$ then for all $t \in [0,1]$,
\begin{equation}
u((1-t)x + ty) = (1-t) u(x) + t u(y). \label{eq:uisaffineonrays}
\end{equation}
Further, $u$ is differentiable for all  $z \in ]x,y[$, with derivative satisfying
\begin{equation}
\nabla u(z) = \frac{x-y}{|x-y|}. \label{eq:graduisraydirection}
\end{equation}\label{lem:affineanddifferentiableoninterior}
\end{lemma}

Lemma \ref{lem:graduisraydirectionextends} is an easy extension of Lemma \ref{lem:affineanddifferentiableoninterior} that holds when the function $u$ is also differentiable at the ray endpoints.
\begin{lemma}
If $[x,y]$ is a transport ray of $u$ and $u$ is differentiable at either endpoint then \eqref{eq:graduisraydirection} also holds at that endpoint.\label{lem:graduisraydirectionextends}
\end{lemma}
\begin{proof}
The proof is contained in the proof of Corollary 3.8 from \cite{santambrogio2015optimal}.
\end{proof}
It was first shown in \cite{caffarelli2002constructing} that $\nabla u$ is actually Lipschitz away from the endpoints of transport rays. To state this result formally we must define the distance between a point $z$ in a transport ray and the endpoints of that ray.
\begin{definition}
Given $u\in \onelip(\Omega)$, set
\begin{equation}
D = \{ x \in \Omega \mid \nabla u(x) \text{ exists}\}.
\end{equation}
Define $\ell_+, \ell_- : D \rightarrow [0, \infty)$ by
\begin{equation}
\ell_\pm(z) = \sup\{ t\in [0,\infty) \mid  u(z \pm t \nabla u(z)) - u(z) = \pm t\}.
\end{equation}
\end{definition}
Note that if $\ell_+(z) >0$ or $\ell_-(z) >0$, $z$ must be in at least one transport ray, and if both are positive $z$ is on the interior of a unique ray by Corollary 3.8 from \cite{santambrogio2015optimal}.

The proof of the following result on the Lipschitz regularity of $\nabla u$ away from the ray endpoints is contained in the proof of Lemma 22 of \cite{caffarelli2002constructing}, however there it is stated with sufficiently specialized notation that it may be helpful to provide a restatement and reproof here.
\begin{proposition}
Let $u \in \onelip(\Omega)$ and define, for $j \in \mathbb{N}$, 
\begin{equation}
A_j = \{ z \in D \mid \min(\ell_-(z), \ell_+(z)) > 1/j\}.
\end{equation}
Then $z \mapsto \nabla u(z)$ is Lipschitz on $A_j$ with constant $4j$.\label{prop:Lipschitzgradient}
\end{proposition}
\begin{proof}
Let $z, z' \in A_j$. Note that since $u \in \onelip(\Omega)$, if $|z-z'| \geq \frac{1}{2j}$ then we have the trivial Lipschitz bound
\begin{equation}
|\nabla u(z) - \nabla u(z')| \leq 2 \leq 4 j |z-z'|.\label{eq:triviallipbound}
\end{equation}
Hence, we focus on the case $|z-z'| < \frac{1}{2j}$. In this case the Lipschitz constant of $u$ allows us to bound the variation in $u$ on these points;
\begin{equation}
|u(z) - u(z')| < \frac{1}{2j}.\label{eq:boundonuvariation}
\end{equation}
Set $w' = z' + (u(z) - u(z')) \nabla u(z')$. By \eqref{eq:boundonuvariation}, we have that $w'$ and $z'$ are on the same transport ray. Indeed, $w'$ is at most $\frac{1}{2j}$ away from $z'$, and $z'$ is at least $\frac{1}{j}$ from the endpoints of the transport ray it is contained in by definition of $A_j$.

Since $w'$ and $z'$ are on the same transport ray, $w'$ lies on the same level set of $u$ as $z$. Indeed, using Lemma \ref{lem:affineanddifferentiableoninterior},
\begin{align}
u(w') &= u(z' + (u(z) - u(z'))\nabla u(z')),\\
&= u(z') + (u(z) - u(z')),\\
&= u(z).
\end{align}
Since both $w'$ and $z$ are interior points of their transport rays and exist on the same level set of $u$ we can then invoke Lemma 16 from \cite{caffarelli2002constructing} to obtain that
\begin{equation}
|\nabla u(w') - \nabla u(z) | \leq \frac{1}{\sigma}|w'-z|,
\end{equation}
where $\sigma$ is the minimal distance from $w'$ or $z$ to the endpoints of its transport ray; by construction this is at least $\frac{1}{2j}$. Hence,
\begin{equation}
|\nabla u(w') - \nabla u(z) | \leq 2j|w'-z|.
\end{equation}
Given that $\nabla u(w') = \nabla u(z')$, we therefore have
\begin{equation}
|\nabla u(z') - \nabla u(z) | \leq 2j|z'-z| + 2j|w'-z'|.\label{eq:preliminlipbound}
\end{equation}
Estimating the last term,
\begin{align}
|w'-z'| &= |u(w') - u(z')|,\\
&= |u(z) - u(z')| ,\\
&\leq |z-z'|,
\end{align}
whence \eqref{eq:preliminlipbound} gives us
\begin{equation}
|\nabla u(z') - \nabla u(z) | \leq 4j |z'-z|
\end{equation}
for all $|z'-z| \leq \frac{1}{2j}$. Combining this with \eqref{eq:triviallipbound}, we obtain that $z \mapsto \nabla u 
(z)$ is Lipschitz on $A_j$ with constant $4j$. 
\end{proof}

\begin{proof}[\textit{Proof of Lemma \ref{lem:graduistransportdirection}}]
We need to specify the set where \eqref{eq:transportdirection} holds. We start by defining the set $C$ as all points $x$ such that the pair $(x,T_0(x))$ saturates the Lipschitz inequality of $u_0$,
\begin{equation}
C = \{ x \in \Omega \mid  u_0(x) - u_0(T_0(x)) = |x-T_0(x)|\}. \label{eq:Cdef}
\end{equation}
$C$ has full $\mu$ measure (meaning $\mu(C^c) = 0$) by the following argument: it is a standard consequence of Kantorovich's duality theorem (\eg the discussion following equation (3.2) in \cite{santambrogio2015optimal}) that
\begin{equation}
\spt((Id,T_0)_\# \mu) \subset \{ (x,y) \in \Omega^2 \mid u_0(x) - u_0(y) = |x-y|\}. \label{eq:sptcontainedintrays}
\end{equation}
Letting $\Gamma_{T_0}(C^c) = \{ (x,T_0(x)) \mid x \in C^c\}$, \eqref{eq:sptcontainedintrays} gives us that
\begin{equation}
(Id, T_0)_\# \mu (\Gamma_{T_0}(C^c)) = 0,
\end{equation}
but $(Id, T_0)^{-1}(\Gamma_{T_0}(C^c)) = C^c$, so $\mu(C^c) = 0$. Hence $C$ has full $\mu$ measure. Note also that $D$ has full $\mu$ measure by Rademacher's Theorem.

We claim that for all $x \in C \cap D$, \eqref{eq:transportdirection} holds. Indeed, suppose $x \neq T_0(x)$. Then by definition of $C$, the pair $(x,T_0(x))$ is in a transport ray. By Lemma \ref{lem:affineanddifferentiableoninterior} or Lemma \ref{lem:graduisraydirectionextends}, we obtain that $\nabla u_0(x)$ is aligned with the direction of that ray, and hence by the definition of $C$ again,
\begin{equation}
\nabla u_0(x) = \frac{x-T_0(x)}{|x-T_0(x)|}.
\end{equation}
Thus, \eqref{eq:transportdirection} holds on $C \cap D$, and $\mu((C\cap D)^c) = 0$.
\end{proof}

\begin{proof}[\textit{Proof of Theorem \ref{prop:propertiesofnextmu} parts \ref{item:u0iskantpot} and \ref{item:exactw1value}}]

Both properties will be proved simultaneously by Kantorovich duality. Since $u_0\in \onelip(\Omega)$, Kantorovich duality gives
\begin{align}
W_1(\tilde{\mu},\nu) &\geq \int_\Omega u_0(x) d\tilde{\mu} - \int_\Omega u_0(y) d\nu,\\
&=\int_\Omega u_0(f_0(x)) d\mu - \int_\Omega u_0(y) d\nu,\\
&=\int_{C\cap D} u_0\left(x - \eta_0 \frac{x-T_0(x)}{|x-T_0(x)|} \right)d\mu, \nonumber\\
&\quad- \int_\Omega u_0(y) d\nu.
\end{align}
The last line follows because \eqref{eq:stepsizelessthanminimallength} gives that $x \neq T_0(x)$ $\mu$-almost everywhere, and hence by Lemma \ref{lem:graduistransportdirection},
\begin{equation}
\nabla u_0(x) = \frac{x-T_0(x)}{|x-T_0(x)|}
\end{equation}
on $C\cap D$, which was proven to be a full $\mu$ measure set in the proof of Lemma \ref{lem:graduistransportdirection}. By \eqref{eq:stepsizelessthanminimallength} and Lemma \ref{lem:affineanddifferentiableoninterior} we get that $\mu$-almost everywhere in $C$,
\begin{equation}
u_0\left(x - \eta_0 \frac{x-T_0(x)}{|x-T_0(x)|}\right) = u_0(x) - \eta_0.
\end{equation}
Substituting this into our calculation, we obtain
\begin{align}
W_1(\tilde{\mu},\nu) &\geq \int_{C \cap D} (u_0(x) - \eta_0)d\mu - \int_\Omega u_0 d\nu,\\
&= W_1(\mu, \nu) - \eta_0, \label{eq:W_1lowerbound}
\end{align}
the last equality holding because $u_0$ is a Kantorovich potential for the pair $(\mu,\nu)$.

We will now show that the right hand side of \eqref{eq:W_1lowerbound} is also an upper bound on $W_1(\tilde{\mu},\nu)$. Observe that the plan
\begin{equation}
\gamma:= (f_0, T_0)_\# \mu
\end{equation}
is in $\Pi(\tilde{\mu},\nu)$. Therefore,
\begin{align}
W_1(\tilde{\mu}, \nu) &\leq \int_\Omega |x-y| d(f_0, T_0)_\# \mu, \\
&= \int_\Omega |f_0(x) - T_0(x)| d\mu,\\
&= \int_\Omega |x - T_0(x) - \eta_0 \frac{x-T_0(x)}{|x-T_0(x)|}|d\mu.
\end{align}
In the last line we have again used Lemma \ref{lem:graduistransportdirection}. Using \eqref{eq:stepsizelessthanminimallength} again, we obtain that
\begin{equation}
W_1(\tilde{\mu}, \nu) \leq \int_\Omega |x-T_0(x)| d\mu - \eta_0 = W_1(\mu,\nu)  - \eta_0, \label{eq:W_1upperbound}
\end{equation}
where here we have used optimality of $T_0$. Equality of the lower bound in \eqref{eq:W_1lowerbound} and the upper bound in \eqref{eq:W_1upperbound} mean that the inequalities are actually equalities, proving that $u_0$ is a Kantorovich potential for $W_1(\tilde{\mu}, \nu)$ and that $W_1(\tilde{\mu},\nu) = W_1(\mu,\nu) - \eta_0$.
\end{proof}

We are now ready to prove part \ref{item:densityformula} of Theorem \ref{prop:propertiesofnextmu}. Our proof is just an integration by substitution to calculate the density $\tilde{\rho}$ of $\tilde{\mu}$ given the density $\rho$ of $\mu$, but we must be careful because our change of variables can be irregular near the ends of transport rays. The main idea is to remove $\mu$ negligible sets from $\Omega$ until we can apply the Area Formula to compute $\tilde{\rho}$. 

\begin{proof}[\textit{Proof of Theorem \ref{prop:propertiesofnextmu} part \ref{item:densityformula}}]
Let $E$ be a Lebesgue measurable set. Then with $\rho$ as the density of $\mu$ with respect to $\mathcal{L}_d$,
\begin{align}
\tilde{\mu}(E) &= \int_{f_0^{-1}(E)} \rho(x) dx.
\end{align}
By \eqref{eq:stepsizelessthanminimallength} there exists $\epsilon>0$ such that $\eta_0 + \epsilon < \ell_0(T_0)$. For this $\epsilon$, define
\begin{equation}
F = \{ x \in \Omega \mid |x - T_0(x)| > \eta_0 + \epsilon\}.\label{eq:Fdeff}
\end{equation}
$F$ has full $\mu$ measure by \eqref{eq:stepsizelessthanminimallength}. Set also $B = \ell_+^{-1}(0)$, which is the set of points such that $u_0$ fails to saturate its Lipschitz constant for any amount of movement in the direction $\nabla u_0$. We claim that
\begin{equation}
P:= B^c \cap C \cap F
\end{equation}
has full $\mu$ measure. We are interested in $P$ because for all $x \in P$, $x$ must be in a transport ray which has descending length at least $\eta$ by virtue of being in $C \cap F$, and must not be at the start of that ray by virtue of being in $B^c$. Using Proposition \ref{prop:Lipschitzgradient}, this will be the set where our change of variables is sufficiently regular.

Since $C$ and $F$ both have full $\mu$ measure, to show that $P$ has full measure we need only show that 
\begin{equation}
\mu(B \cap C \cap F) = 0.\label{eq:munnegligble}
\end{equation}
If $x \in C \cap F$, however, $x$ must be in a transport ray of $u_0$, and if in addition $x \in B$, then $x$ must be the upper endpoint of that ray. The set of endpoints of transport rays has Lebesgue measure $0$ by Lemma 25 of \cite{caffarelli2002constructing}, and hence $\mu$ measure $0$ since $\mu\ll \mathcal{L}_d$. Thus \eqref{eq:munnegligble} holds, and hence $P$ has full $\mu$ measure. Thus we obtain
\begin{equation}
\tilde{\mu}(E) = \int_{P \cap f_0^{-1}(E)} \rho(x) dx.
\end{equation}

We can write $P \cap f_0^{-1}(E)$ as the image of $g_0$ on a certain set. In particular
\begin{equation}
P \cap f_0^{-1}(E) = g_0(f_0(P) \cap E). \label{eq:setasimageset}
\end{equation}
To prove \eqref{eq:setasimageset}, note that on $P$ we have $g_0(f_0(x)) = x$. Indeed, for $x \in P$, both $x$ and $f_0(x)$ are on the interior of the transport ray containing $x$ and $T_0(x)$, and hence by Lemma \ref{lem:affineanddifferentiableoninterior}, $\nabla u_0(x)$ and $\nabla u_0(f_0(x))$ exist and are equal,
which implies $g_0(f_0(x)) = x$. So,
\begin{equation}
P \cap f_0^{-1}(E) = g_0(f_0(P \cap f_0^{-1}(E))), \label{eq:glambflamb=I}
\end{equation}
from which \eqref{eq:setasimageset} follows.

As a consequence of the Area Formula (Theorem 3.2.5. from \cite{federer2014geometric}), if $\rho:\RR^d \rightarrow \RR$ is Lebesgue measurable, $A \subset \RR^d$ is Lebesgue measurable, and $g: A \rightarrow \RR^d$ injective and Lipschitz on $A$, we have
\begin{equation}
\int_{g(A)} \rho(y) dy = \int_A \rho(g(x))|Dg(x)| dx,
\end{equation}
where $|Dg(x)|$ denotes the absolute value of the determinant of the Jacobian of $g$. Thus, provided $f_0(P) \cap E$ is measurable and $g_0$ is injective and Lipschitz on this set, the area formula gives
\begin{equation}
\int_{g_0(f_0(P) \cap E)} \rho(x) dx = \int_{f_0(P) \cap E} \rho(g_0(x))|Dg_0(x)| dx.
\end{equation}
Assuming we have this formula, we obtain
\begin{equation}
\tilde{\mu}(E) = \int_E 1_{f_0(P)}(x)\rho(g_0(x))|D g_0(x)| dx.
\end{equation}
Hence, $\tilde{\mu} \ll \mathcal{L}_d$ with integrable density
\begin{equation}
\tilde{\rho}(x) = 1_{f_0(P)}(x)\rho(g_0(x))|D g_0(x)|.
\end{equation}
Evaluating this formula at $f_0(x)$ for $x\in P$, we get \eqref{eq:densityformula}, and part \ref{item:densityformula} of Theorem \ref{prop:propertiesofnextmu} is proved. Hence, it only remains to prove that $f_0(P) \cap E$ is measurable, and that $g_0$ is injective and Lipschitz on this set. These facts are proved in Lemmas \ref{lem:imagesetismeasurable} and \ref{lem:g_lambdainjectiveandlipschitz}.
\end{proof}
\begin{lemma}
The set $f_0(P)$ is measurable. \label{lem:imagesetismeasurable}
\end{lemma}
\begin{proof}
By definition of $B$ we have $\ell_+ >0$ on $P$. So, defining
\begin{align}
B_j &= \ell_+^{-1}\left(\left[\frac{1}{j}, \infty\right)\right), \quad j \geq 1,
\end{align} 
we obtain that $P = \bigcup_{j=1}^\infty P \cap B_{j}$. Hence,
\begin{equation}
f_0(P) = \bigcup_{j=1}^\infty f_0(P \cap B_j).
\end{equation}
For all $j\geq 1$, $P \cap B_j = C \cap F \cap B_j$ is measurable because $C$ and $F$ are measurable and $\ell_+$ is an upper semi-continuous function on its domain (see Lemma \ref{lem:ell+uppersemicont}). Hence, if $f_0$ is injective and Lipschitz on $P \cap B_j$, then Proposition 262E of \cite{fremlin2000measure} will give us that $f_0(P \cap B_j)$ is measurable for all $j$, proving the lemma.

In fact, $f_0$ is obviously injective on $P$ since $g_0(f_0(x)) = x$ for all $x \in P$. That $f_0$ is Lipschitz on $P \cap B_j$ follows from Proposition \ref{prop:Lipschitzgradient}. Indeed, by definition of $P$, $\ell_-(x) > \eta_0$ on $P$, and by definition of $B_j$, $\ell_+(x) \geq 1/j$. Hence, there exists $k$ such that $P \cap B_j \subset A_k$, and thus $f_0$ is Lipschitz on $P \cap B_j$.
\end{proof}
\begin{lemma}
The function $\ell_+$ is upper semi-continuous on its domain.\label{lem:ell+uppersemicont}
\end{lemma}
\begin{proof}
The proof is similar to that of Lemma 24 in \cite{caffarelli2002constructing}, but the setting is different, so we provide the details here. Let $(x_n)_{n=1}^\infty \subset D$ converge to a point $x \in D$. If
\begin{equation}
\limsup_{n} \ell_+(x_n) = 0,
\end{equation}
then $\ell_+(x) \geq \limsup_n \ell_+(x_n)$ trivially. If $m_0:=\limsup_n \ell_+(x_n) >0$, we can extract a subsequence and relabel so that $m_0=\lim_n \ell_+(x_n) >0$. Since $(\nabla u_0(x_n))_n$ is a bounded sequence, we may extract a further subsequence and relabel again so that $\lim_{n\rightarrow \infty} \nabla u_0(x_n) = v$ for some unit vector $v$. The continuity of $u_0$ then implies
\begin{align}
u_0(x + m_0 v) - u_0(x) &= \lim_{n\rightarrow \infty} u_0(x_n + \ell_+(x_n) \nabla u_0(x_n)) \nonumber\\
&\qquad - u_0(x_n),\\
&= \lim_{n \rightarrow \infty} \ell_+(x_n),\\
&= m_0.\label{eq:ell+xisatleastell0}
\end{align}
Since $x \in D$ and $u_0\in\onelip(\Omega)$ Lemma \ref{lem:affineanddifferentiableoninterior} or Lemma \ref{lem:graduisraydirectionextends} implies that we must have $\nabla u_0(x) = v$. As such, \eqref{eq:ell+xisatleastell0} implies $\ell_+(x) \geq m_0$, proving upper semi-continuity. 
\end{proof}

\begin{lemma}
The function $g_0$ is injective and Lipschitz on $f_0(P)$. \label{lem:g_lambdainjectiveandlipschitz}
\end{lemma}
\begin{proof}
By definition of $f_0$ and $P$ we obtain that
\begin{equation}
\inf_{x \in P} \min(\ell_-(f_0(x)), \ell_+(f_0(x))) >0.
\end{equation}
Thus, there exists $j$ such that $f_0(P) \subset A_j$, and hence by Lemma \ref{prop:Lipschitzgradient} $g_0$ is Lipschitz on $f_0(P)$.

To see that $g_0$ is injective on $f_0(P)$, let $x_1, x_2 \in P$ such that
\begin{equation}
g_0(f_0(x_1)) = g_0(f_0(x_2)).
\end{equation}  
Since $g_0(f_0(x)) = x$ for $x \in P$, we immediately obtain $f_0(x_1) = f_0(x_2)$, and hence $g_0$ is injective on $f_0(P)$. 
\end{proof}
Part \ref{item:optimalmap}  of Theorem \ref{prop:propertiesofnextmu} now follows easily.
\begin{proof}[\textit{Proof of Theorem \ref{prop:propertiesofnextmu} part \ref{item:optimalmap}}]
Since $g_0 \circ f_0 = I$ on $P$ and $\mu(P^c) = 0$, we obtain that
\begin{align}
(T_{1})_\# \tilde{\mu} &= (T_0 \circ g_0)_\# (f_0)_\# \mu,\\
&= (T_0)_\# \mu,\\
&= \nu.
\end{align}
Hence $T_{1}$ is an admissible map for computing $W_1(\tilde{\mu},\nu)$. Its cost is
\begin{align}
\int_\Omega |x-T_{1}(x)| d\tilde{\mu} &= \int_\Omega |f_0(x) - T_0(x)| d\mu,\\
&= \int_\Omega |x- T_0(x) - \eta_0 \frac{x-T_0(x)}{|x-T_0(x)|}|d\mu.
\end{align}
Applying \eqref{eq:stepsizelessthanminimallength} we get
\begin{equation}
\int_\Omega |x-T_{1}(x)| d\tilde{\mu}= \int_\Omega |x-T_0(x)|d\mu - \eta_0 = W_1(\tilde{\mu}, \nu), 
\end{equation}
and hence $T_{1}$ is optimal. In the course of proving this we have shown that $\mu$ almost everywhere
\begin{equation}
|f_0(x)-T_{1}(f_0(x))| = |x-T_0(x)| - \eta_0.
\end{equation}
As such,
\begin{equation}
\mu(\{ x \mid |f_0(x) - T_{1}(f_0(x))| < \ell_0(T_0) - \eta_0\}) = 0.
\end{equation}
Hence
\begin{equation}
\tilde{\mu}(\{ x \mid |x - T_{1}(x)| < \ell_0(T_0) - \eta_0\}) = 0,
\end{equation}
which implies $\ell_0(T_{1}) \geq \ell_0(T_0) - \eta_0$.
\end{proof}
Before we proceed with the proof of Theorem \ref{prop:approximateconvergence}, we present the following lemma, which shows that provided $\mu$ has a density and the minimal transport distance is positive, the gradients of Kantorovich potentials for the pair $(\mu,\nu)$ agree $\mu$-almost everywhere. Thus, the measure $\mu_n$, which in principle depends on a choice of Kantorovich potential for the pair $(\mu_{n-1},\nu)$, is actually independent of this choice provided these conditions hold. 
\begin{lemma}
Suppose that $\mu \ll \mathcal{L}_d$, $T_0$ is an optimal transport map for the pair $(\mu,\nu)$, and $u_0, \tilde{u}_0$ are Kantorovich potentials for the same pair. If $\ell_0(T_0) >0$, then $\mu$-almost everywhere \label{lem:equalityofgrads}
\begin{equation}
    \nabla u_0(x) = \nabla \tilde{u}_0(x).\label{eq:equalityofgrads}
\end{equation}
\end{lemma}
\begin{proof}
Since $\ell_0(T_0) >0$, we have that $x \neq T_0(x)$ for $\mu$ almost all $x$. Applying Lemma \ref{lem:graduistransportdirection} separately to $u_0$ and $\tilde{u}_0$, we obtain \eqref{eq:equalityofgrads} $\mu$-almost everywhere.
\end{proof}
\begin{proof}[Proof of Theorem \ref{prop:approximateconvergence}]
Let $u_0$ be a Kantorovich potential for the pair $(\mu_0,\nu)$. We begin by claiming that for all $n \in \{ 0, 1, \ldots, N(\theta))$ we have
\begin{enumerate}[i.]
\item $u_0$ is a Kantorovich potential for the pair $(\mu_n, \nu)$,
\item $W_1(\mu_n, \nu) = (1-\theta)^n W_1(\mu_0, \nu)$,
\item $\mu_n \ll \mathcal{L}_d$, and,
\item for $n>0$ and setting $T_n:= T_{n-1} \circ g_{n-1}$ with
\begin{equation}
g_{n-1}(x) = \begin{cases}
x + \eta_{n-1} \nabla u_0(x) &\quad {\nabla u_0(x) \text{ exists}},\\
x &\quad \text{else,}
\end{cases} 
\end{equation}
$T_n$ is an optimal transport map from $\mu_n$ to $\nu$ satisfying
\begin{equation}
\ell_0(T_n) \geq \ell_0(T_0)- W_1(\mu_0,\nu)(1-(1-\theta)^{n}).
\end{equation}
\end{enumerate}
Points (i - iv) will be proven simultaneously by induction. The base case is true by definition. For the inductive step, suppose that (i - iv) hold for index $n-1$. We will then check that the conditions of Theorem \ref{prop:propertiesofnextmu} hold with $\mu_{n-1}$ as $\mu$, $\mu_n$ as $\tilde{\mu}$ and $T_{n-1}$ as the optimal map from $\mu_{n-1}$ to $\nu$. Note that by definition of $N(\theta)$, our inductive assumption guarantees that $\ell_0(T_{n-1}) >0$, and hence via Lemma \ref{lem:equalityofgrads} we have that $\mu_n = (I-\eta_{n-1} \nabla u_0)_\# \mu_{n-1}$. 

Therefore, the only condition in Theorem \ref{prop:propertiesofnextmu} which must be checked is that 
\begin{equation}
\eta_{n-1} < \ell_0(T_{n-1}). \label{eq:conditionneeded}
\end{equation}
But, by our inductive assumption we have that 
\begin{align}
\ell_0(T_{n-1}) &\geq \ell_0(T_0) -  W_1(\mu_0, \nu)(1-(1-\theta)^{n-1}).
\end{align}
Hence, the inequality \eqref{eq:conditionneeded} will hold if
\begin{align}
\ell_0(T_0) - W_1(\mu_0, \nu)(1-(1-\theta)^{n-1}) &> \eta_{n-1}.
\end{align}
Writing $\eta_{n-1} = \theta(1-\theta)^{n-1}W_1(\mu_0,\nu)$ and rearranging, this is equivalent to
\begin{equation}
(1-\theta)^n > 1 - \frac{\ell_0(T_0)}{W_1(\mu_0,\nu)}.
\end{equation}
Taking logarithms in base $1-\theta$ we see that $n$ satisfies this inequality if and only if
\begin{equation}
n < \log_{1-\theta}\left(1 - \ell_0(T_0)/W_1(\mu_0,\nu)\right).\label{eq:ifholds,prop1applicable}
\end{equation}
Hence, if
\begin{equation}
n \leq \lceil \log_{1-\theta}\left(1 - \ell_0(T_0)/W_1(\mu_0,\nu)\right) \rceil -1 = N(\theta),
\end{equation}
we have \eqref{eq:conditionneeded}. Thus we may apply Theorem \ref{prop:propertiesofnextmu}, from which points i and iii are immediate. From \ref{item:exactw1value} in Theorem \ref{prop:propertiesofnextmu} we obtain
\begin{align}
W_1(\mu_n,\nu) &= W_1(\mu_{n-1}, \nu) - \theta W_1(\mu_{n-1}\nu),\\
&= (1-\theta)^n W_1(\mu_0, \nu),
\end{align}
which proves ii.
Finally, iv follows from part \ref{item:optimalmap} from Theorem \ref{prop:propertiesofnextmu}, together with the computation
\begin{align}
\ell_0(T_n) &\geq \ell_0(T_{n-1}) - \eta_{n-1},\\
&\geq \ell_0(T_0) - W_1(\mu_0,\nu)(1 - (1-\theta)^{n-1}) \nonumber \\
&\quad - \theta(1-\theta)^{n-1}W_1(\mu_0, \nu),\\
&=\ell_0(T_0) - W_1(\mu_0,\nu)(1-(1-\theta)^n).
\end{align}
Thus we have proven claims i-iv. In particular \eqref{eq:W1estimate} is proved, as it is claim ii. Finally, since
\begin{align}
W_1(\mu_{N(\theta)},\nu) &= (1-\theta)^{N(\theta)}W_1(\mu_0,\nu),\\
&= \frac{W_1(\mu_0, \nu)}{1-\theta} (1-\theta)^{N(\theta) +1},
\end{align}
and
\begin{equation}
N(\theta) + 1 \geq \log_{1-\theta}\left(1 - \ell_0(T_0)/W_1(\mu_0,\nu)\right), \label{eq:N(theta)bound}
\end{equation}
we obtain that
\begin{equation}
W_1(\mu_{N(\theta)},\nu) \leq \frac{W_1(\mu_0,\nu) - \ell_0(T_0)}{1-\theta},
\end{equation}
as claimed.
\end{proof}
\begin{proof}[Proof of Proposition \ref{prop:advregis1ttc}]
Since $u_0 \in \onelip(\Omega)$, the minimal value of \eqref{prob:advreg} is bounded below by
\begin{equation}
    \min_{x\in \RR^d} \frac{1}{2}|x-x_0|^2 - \eta|x-x_0| + \eta u_0(x_0) = \eta u_0(x_0) - \frac{1}{2}\eta^2.
\end{equation}
The equality above follows by minimizing the one-dimensional function $z \mapsto \frac{1}{2}z^2 - \eta z$ over non-negative $z$, which has minimizer $z = \eta$. By assumption, for $\mu$ almost all $x_0$ we have $|x_0-T_0(x_0)| \geq \eta$, and the segment $[x_0, T_0(x_0)]$ is contained in a transport ray of $u_0$. In addition, for $\mu$ almost all $x_0$ Lemma \ref{lem:graduistransportdirection} gives us that $|\nabla u_0(x_0)| = 1$. Thus, by Lemma \ref{lem:affineanddifferentiableoninterior}, we get
\begin{align}
    \frac{1}{2}|\eta \nabla u_0(x_0)|^2 + \eta u_0(x_0 - \eta \nabla u_0(x_0)) &= \eta u_0(x_0) - \frac{1}{2}\eta^2,
\end{align}
and so $x_0 - \eta \nabla u_0(x_0)$ obtains the minimal value of \eqref{prob:advreg}. For uniqueness, observe that any minimizer $x^*$ distinct from $x_0 - \eta \nabla u_0(x_0)$ must satisfy
\begin{equation}
    u_0(x^*) = u_0(x_0) - |x^*-x_0|,
\end{equation}
and thus $x_0$ must exist at the intersection of at least two transport rays. The set of $x_0$ for which this can occur is negligible, completing the proof.  
\end{proof}

\section{Additional experimental results}
\label{sec:exp_results_appendix}

\subsection{Misalignment experiments}
\label{sec:misalignment_appendix}

\begin{table}[h]
\centering
\begin{tabular}{ccc}
\toprule
Stage in training & InfoGAN & DCGAN  \\
\midrule
Early & $0.307\pm0.0532$ & $0.284\pm0.030$  \\
 & \small{($410.4\pm9.0$)} & \small{($414.3\pm26.8$)} \\
\midrule
Mid & $0.148\pm0.083$ & $0.081\pm0.082$\\
 & \small{($22.9\pm0.5$)} & \small{($40.8\pm7.1$)} \\
\midrule
Late & $0.107\pm0.075$ & $0.079\pm0.084$\\
 & \small{($21.8\pm0.4$)} & \small{($31.1\pm1.8$)} \\
\bottomrule
\end{tabular}
\caption{Misalignment cosine measurements. All results are reported in the form mean $\pm$ standard deviation. The top lines show the statistics of the cosines of the angles between \eqref{eq:actualupdatedirection} and $-\nabla u_0(G_w(z))$, computed over the training of five separate initializations of each model, each with 256 randomly sampled noise inputs $z$. The statistics appearing in parentheses are for FID values computed over the five training runs. The values were computed at three different stages in training, the same as in Table \ref{table:ndp1}}
\label{table:ndp2}
\end{table}

\begin{figure*}[b]
\centering

\begin{subfigure}[t]{.49\textwidth}
  \centering
  \includegraphics[width=0.85\linewidth]{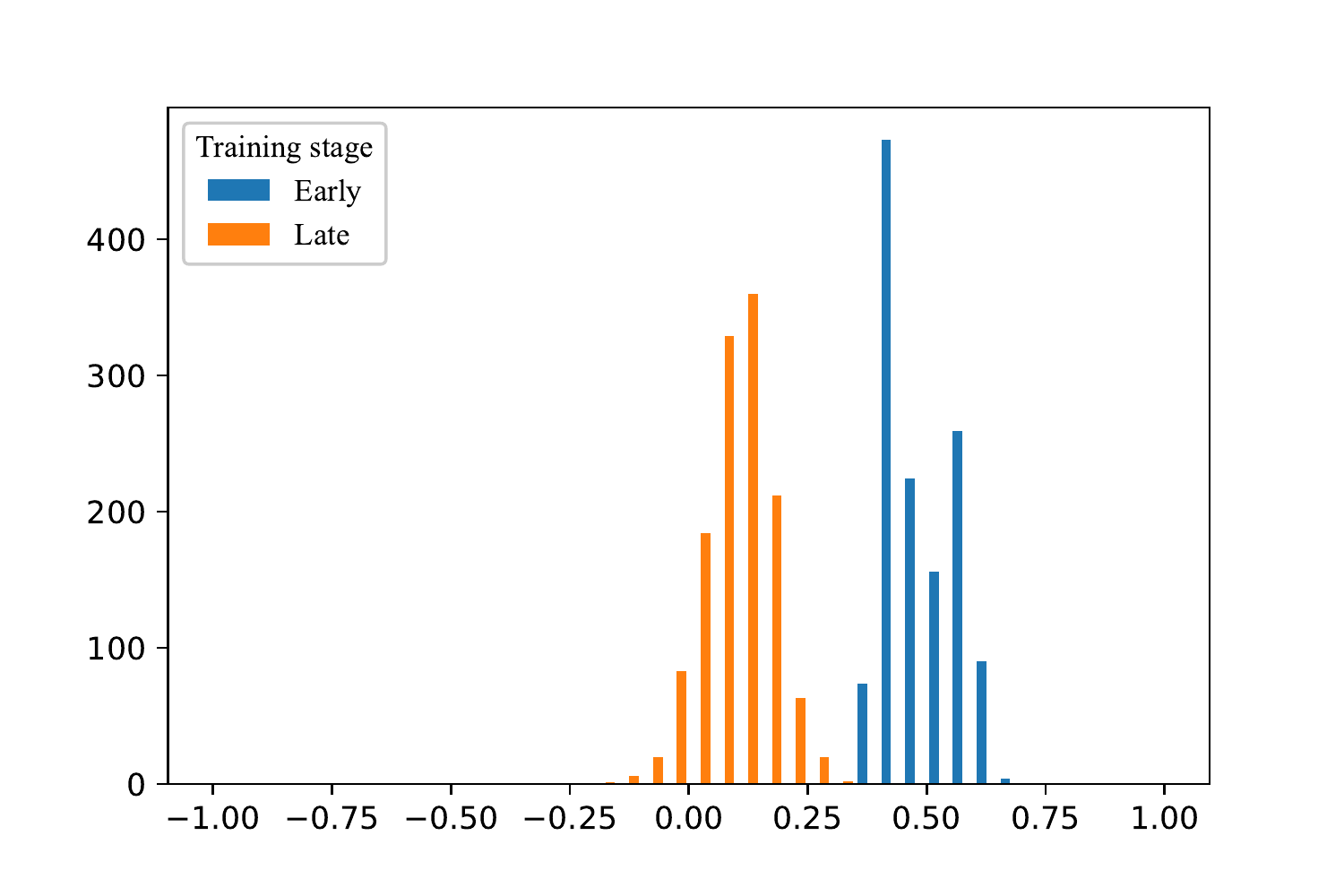}
  \caption{SGD updates vs optimal direction for InfoGAN}
\end{subfigure}
\begin{subfigure}[t]{.49\textwidth}
  \centering
  \includegraphics[width=0.85\linewidth]{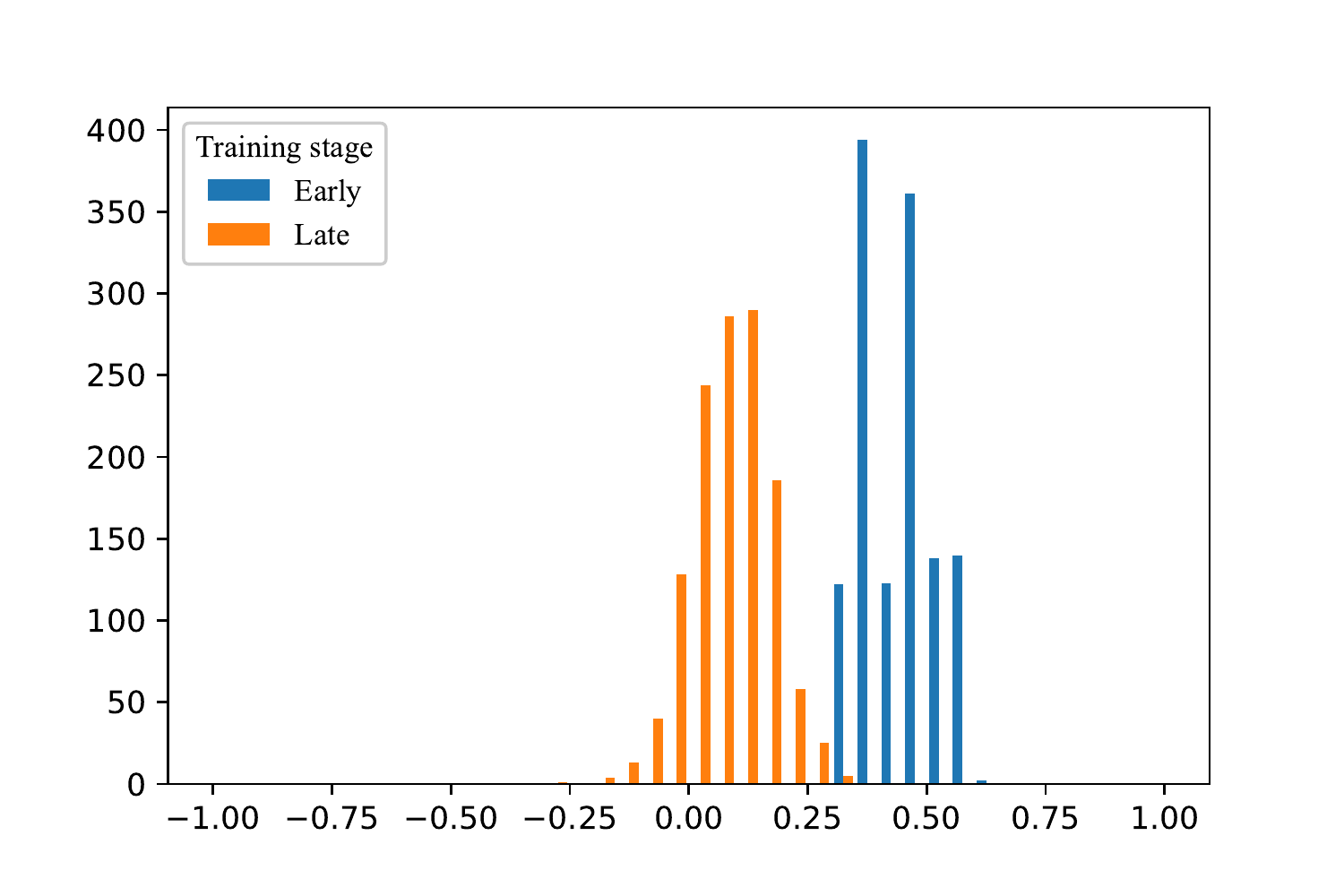}
  \caption{SGD updates vs optimal direction for DCGAN}
\end{subfigure}
\begin{subfigure}[t]{.49\textwidth}
  \centering
  \includegraphics[width=0.85\linewidth]{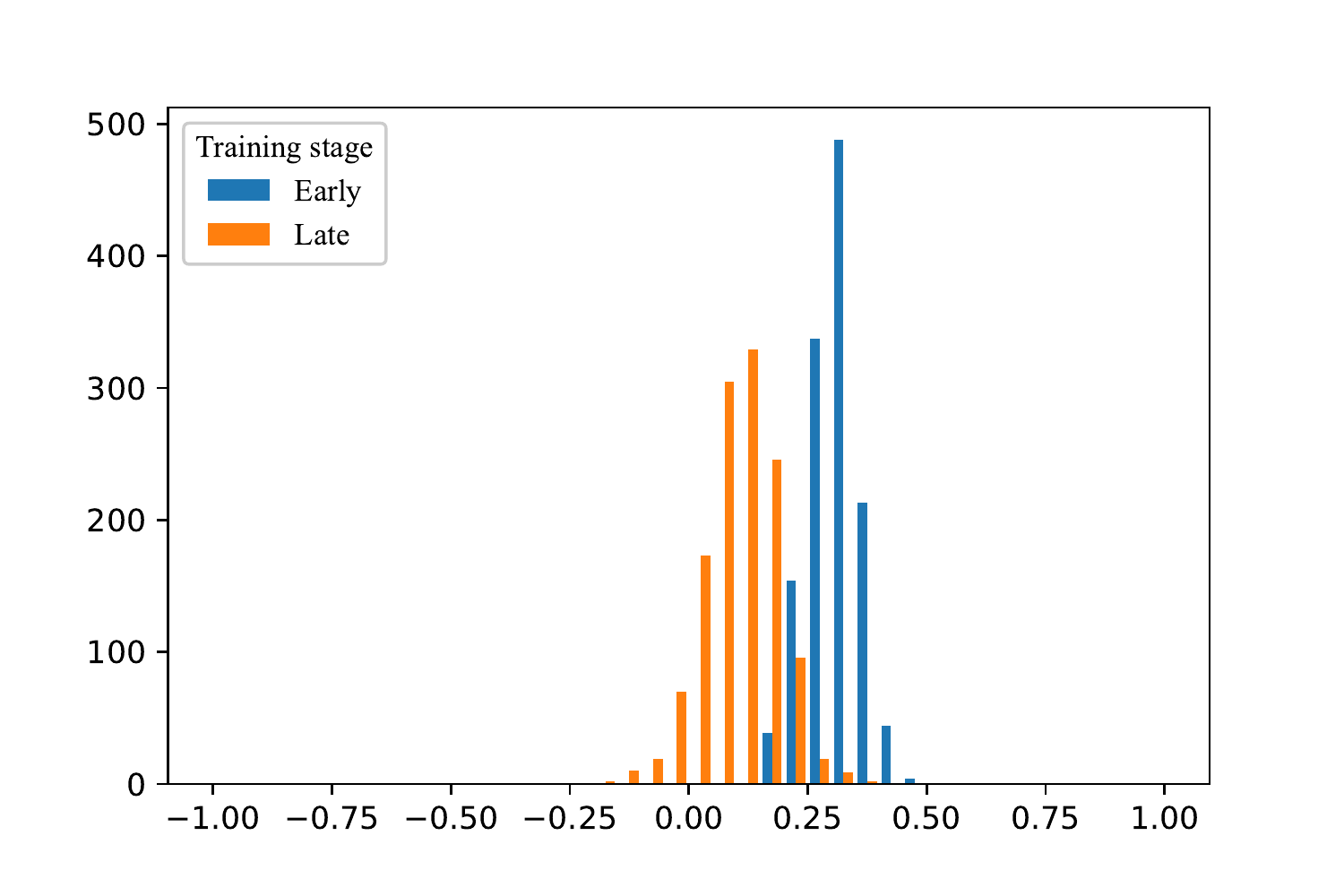}
  \caption{Adam updates vs optimal direction for InfoGAN}
\end{subfigure}
\begin{subfigure}[t]{.49\textwidth}
  \centering
  \includegraphics[width=0.85\linewidth]{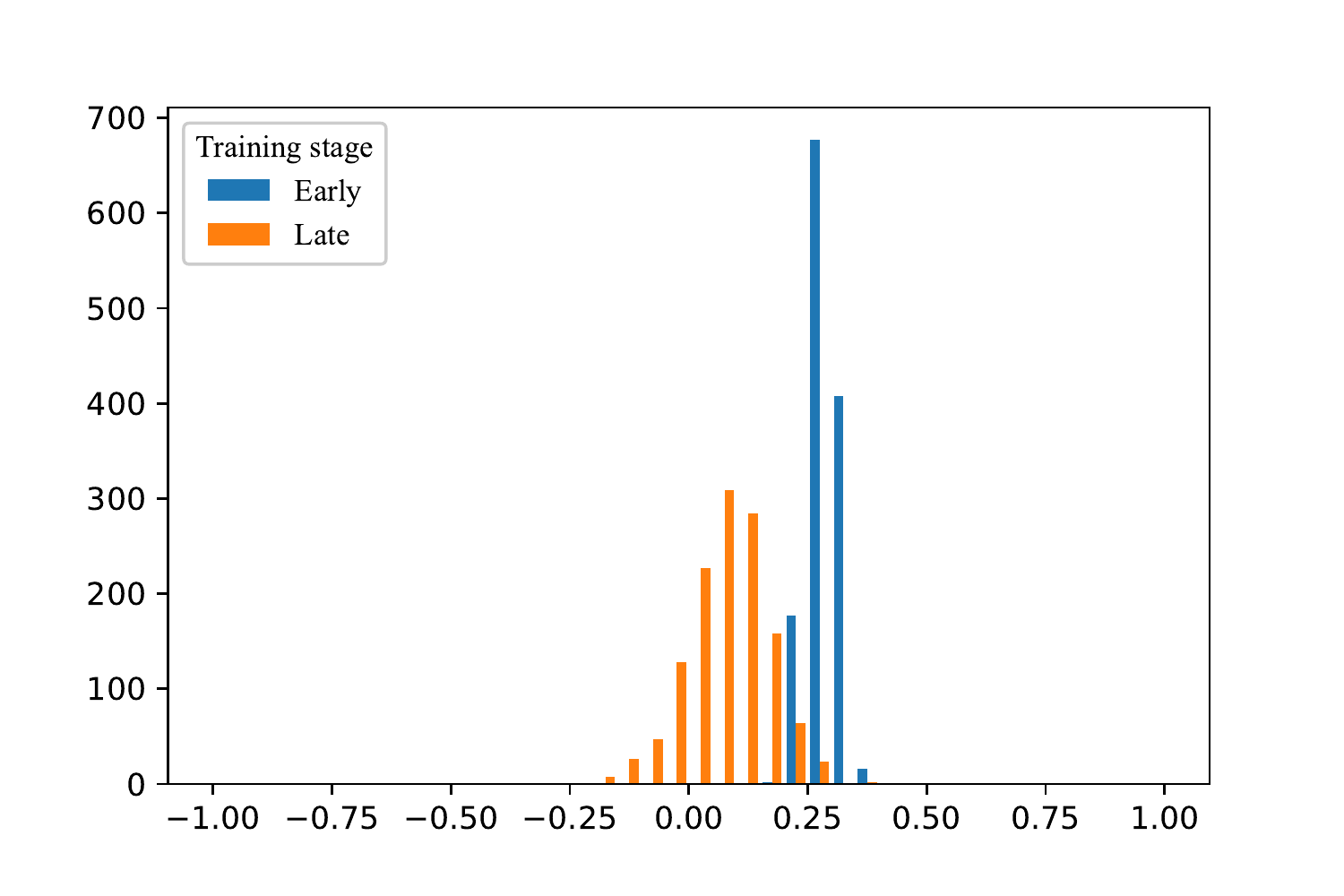}
  \caption{Adam updates vs optimal direction for DCGAN}
\end{subfigure}

\caption{Histograms of the cosines of the angles computed between either SGD or Adam updates of generated samples and the optimal movement direction given by $-\nabla u_0(G_w(z))$. The SGD values corresponds to the statistics in Table \ref{table:ndp1}, while the Adam values correspond to the statistics in Table \ref{table:ndp2}.} 
\label{fig:histograms}
\end{figure*}
In Section \ref{sec:misalignment} we reported on measurements of the misalignment between the directions in which generated samples are updated by SGD steps on the generator and the optimal directions given by $-\nabla u_0(G_w(z))$. Here, we show additional measurements taken during the same experiments, but where the directions of SGD updates are replaced by the directions of Adam updates (Adam was used to train the network). Since Adam is a variation on SGD that takes previous steps into account at every stage, one might expect it to yield larger misalignment. For fixed noise inputs $z$, we computed the difference between generated samples $G_w(z)$ before and after taking one step of an Adam optimizer. We then computed the cosines of the angles between the resulting vectors and the corresponding optimal directions $-\nabla u_0(G_w(z))$. The Adam optimizer used parameters $\beta_1 = 0.5$ and $\beta_2 = 0.999$. The statistics of the misalignment values obtained are shown in Table \ref{table:ndp2}. Comparing to the results in Table \ref{table:ndp1}, we see significantly more misalignment at an early stage in training, but similar values in mid and late stages. 

Figure \ref{fig:histograms} displays histograms of the misalignment cosine values, the statistics of which appear in Table \ref{table:ndp1} and Table \ref{table:ndp2}. To avoid cluttering, we only show the results at early and late stages in training. The histograms clearly show that the movement of generated samples is consistently misaligned with the optimal direction, there being no cosine values above $0.75$. The misalignment is less pronounced at early stages of training, especially when using SGD. Adam and SGD give similar misalignment cosine values at late stages of training, with a small but significant portion of these values being negative, which indicates detrimental movement of generated samples in a direction along which the value of the critic increases.


\subsection{Generated images for MNIST, Fashion MNIST, and CIFAR-10}
\label{sec:generative_appendix}
Figures \ref{fig:generated_mnist}, \ref{fig:generated_fashion}, and \ref{fig:generated_cifar10} include examples of images generated using the trained models from the experiments in Section \ref{sec:ttcvswgangp}. Recall that TTC 1 uses a corresponding untrained generator to create the source distribution $\mu_0$, while TTC 2 uses the standard normal distribution on $\RR^d$ as $\mu_0$. Despite having lower FIDs, the images produced by TTC 1 and TTC 2 do not obviously appear to have improved quality. This is perhaps not surprising, since FID measures the statistics of high-dimensional distributions and thus the modest reductions in FID we have obtained may not be visible to the naked eye for a single minibatch.
\begin{figure*}[b]
    \centering
    \includegraphics[width = 0.65 \textwidth, viewport=50 120 420 215, clip = True]{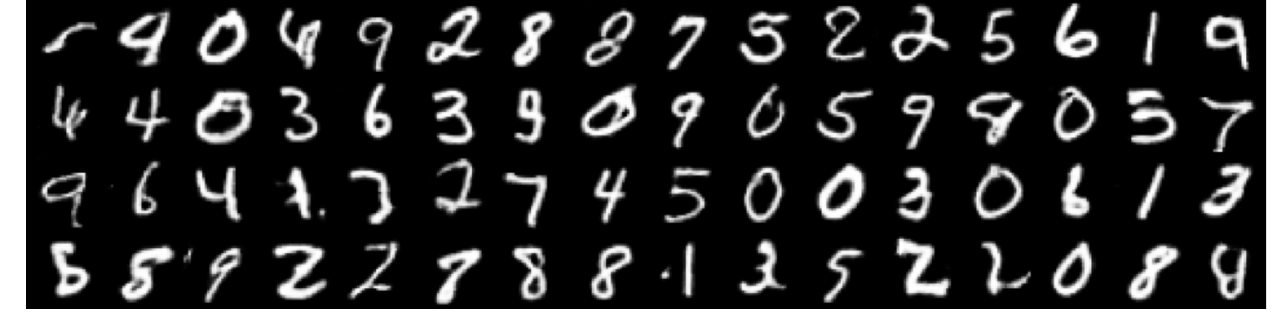}
    \includegraphics[width = 0.65 \textwidth, viewport=50 120 420 215, clip = True]{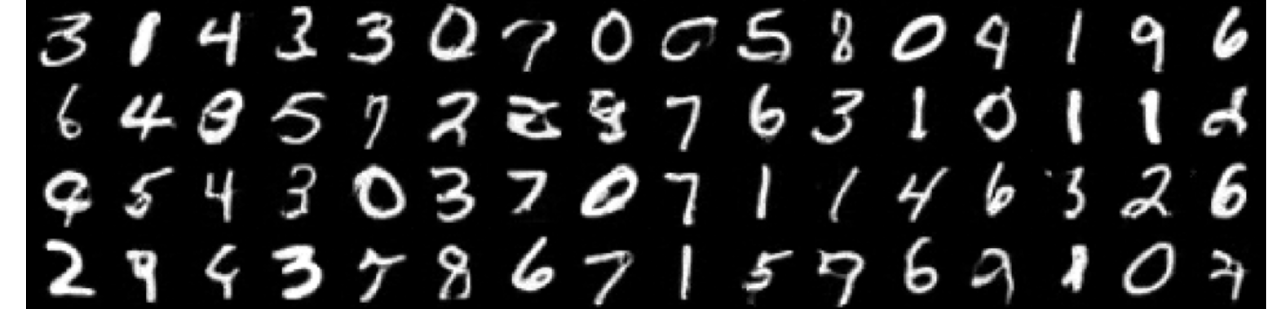}
    \includegraphics[width = 0.65 \textwidth, viewport=50 125 420 215, clip = True]{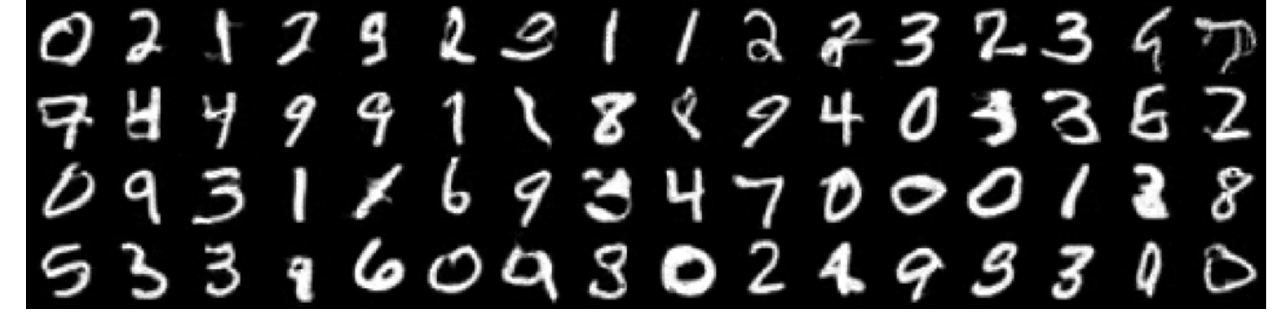}
    \caption{Generated examples for MNIST for a particular training run of each technique. From top to bottom: WGAN-GP (FID 20.8), TTC 1 (FID 18.5), and TTC 2 (FID 18.2).}
    \label{fig:generated_mnist}
\end{figure*}
\begin{figure*}
    \centering
    \includegraphics[width = 0.65 \textwidth, viewport=50 120 420 215, clip = True]{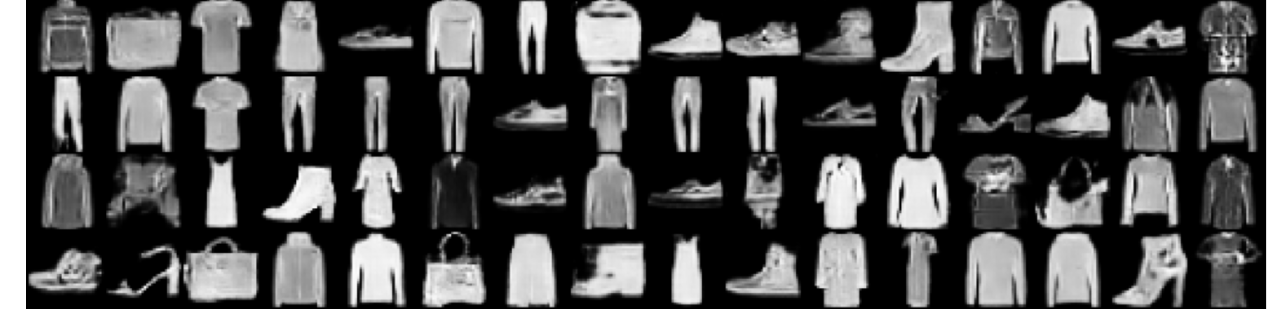}
    \includegraphics[width = 0.65 \textwidth, viewport=50 120 420 215, clip = True]{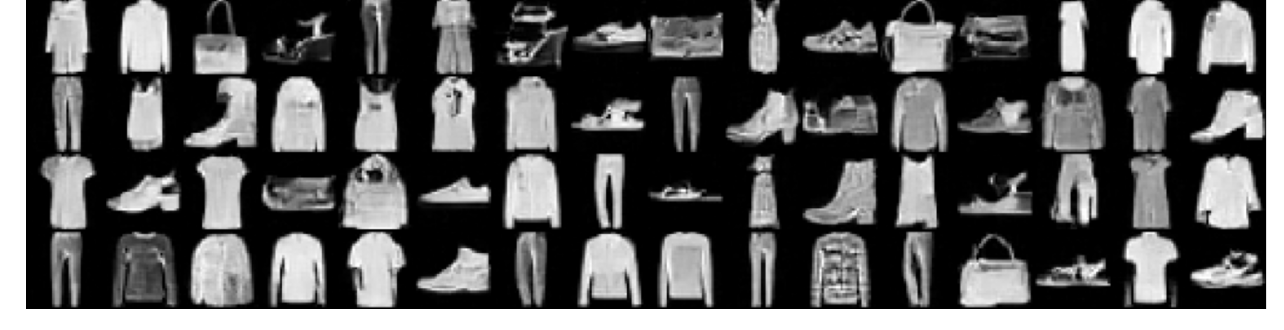}
    \includegraphics[width = 0.65 \textwidth, viewport=50 125 420 215, clip = True]{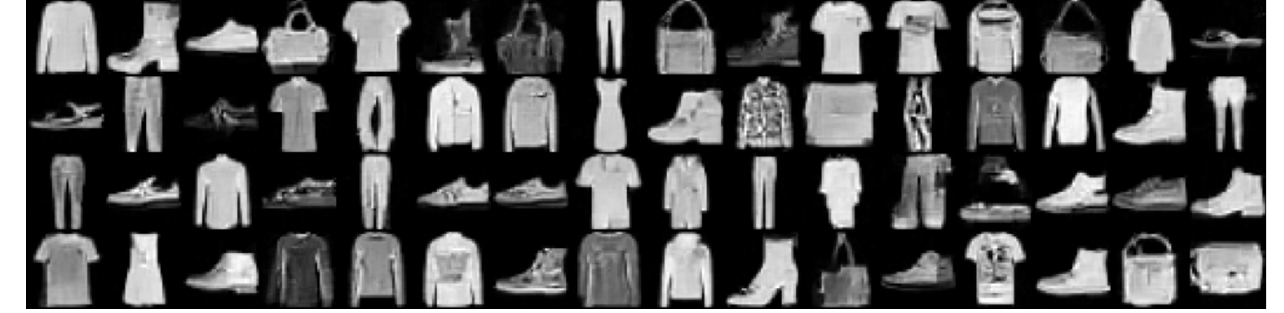}
    \caption{Generated examples for F-MNIST for a particular training run of each technique. From top to bottom: WGAN-GP (FID 25.7), TTC 1 (22.4), and TTC 2 (22.2).}
    \label{fig:generated_fashion}
\end{figure*}
\begin{figure*}
    \centering
    \includegraphics[width = 0.65 \textwidth, viewport=50 120 420 215, clip = True]{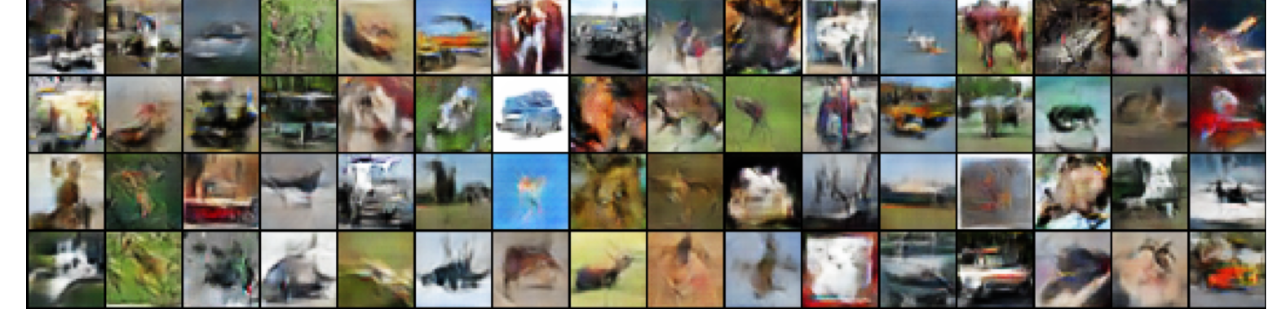}
    \includegraphics[width = 0.65 \textwidth, viewport=50 120 420 215, clip = True]{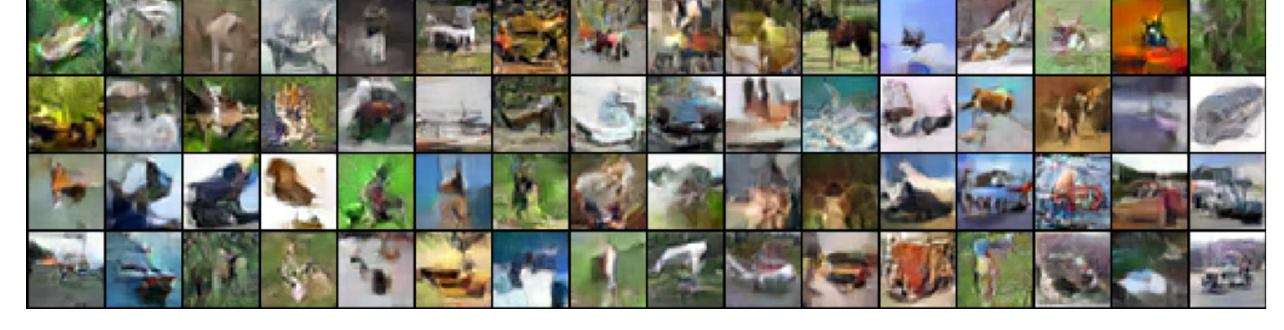}
    \includegraphics[width = 0.65 \textwidth, viewport=50 125 420 215, clip = True]{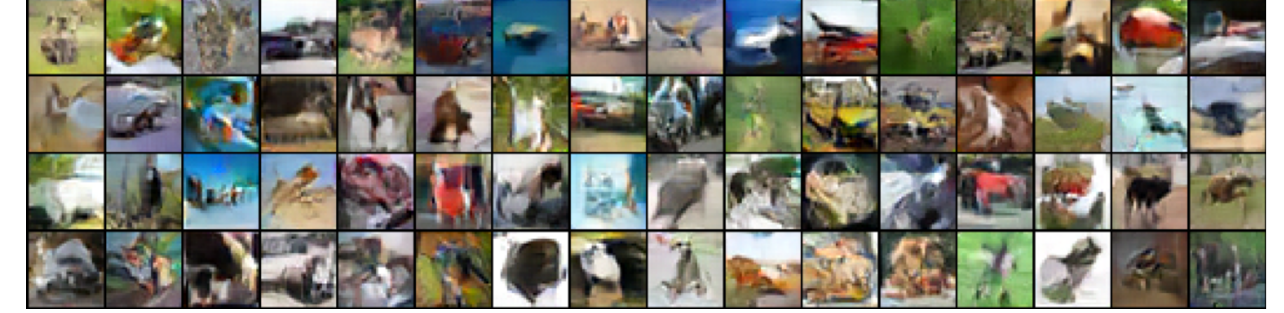}
    \caption{Generated examples for CIFAR10 for a particular training run of each technique. From top to bottom: WGAN-GP (FID 28.9), TTC 1 (FID 27.8), and TTC 2 (FID 26.5).}
    \label{fig:generated_cifar10}
\end{figure*}

\subsection{Additional image translation examples}
\label{app:translation}
Figure \ref{fig:more_monets} contains additional image translation examples, generated using the same sequence of critics and step sizes that were used to generate Figure \ref{fig:translationexample}.
\begin{figure*}
    \centering
    \includegraphics[width = 0.58 \textwidth, viewport=50 75 420 260, clip = True]{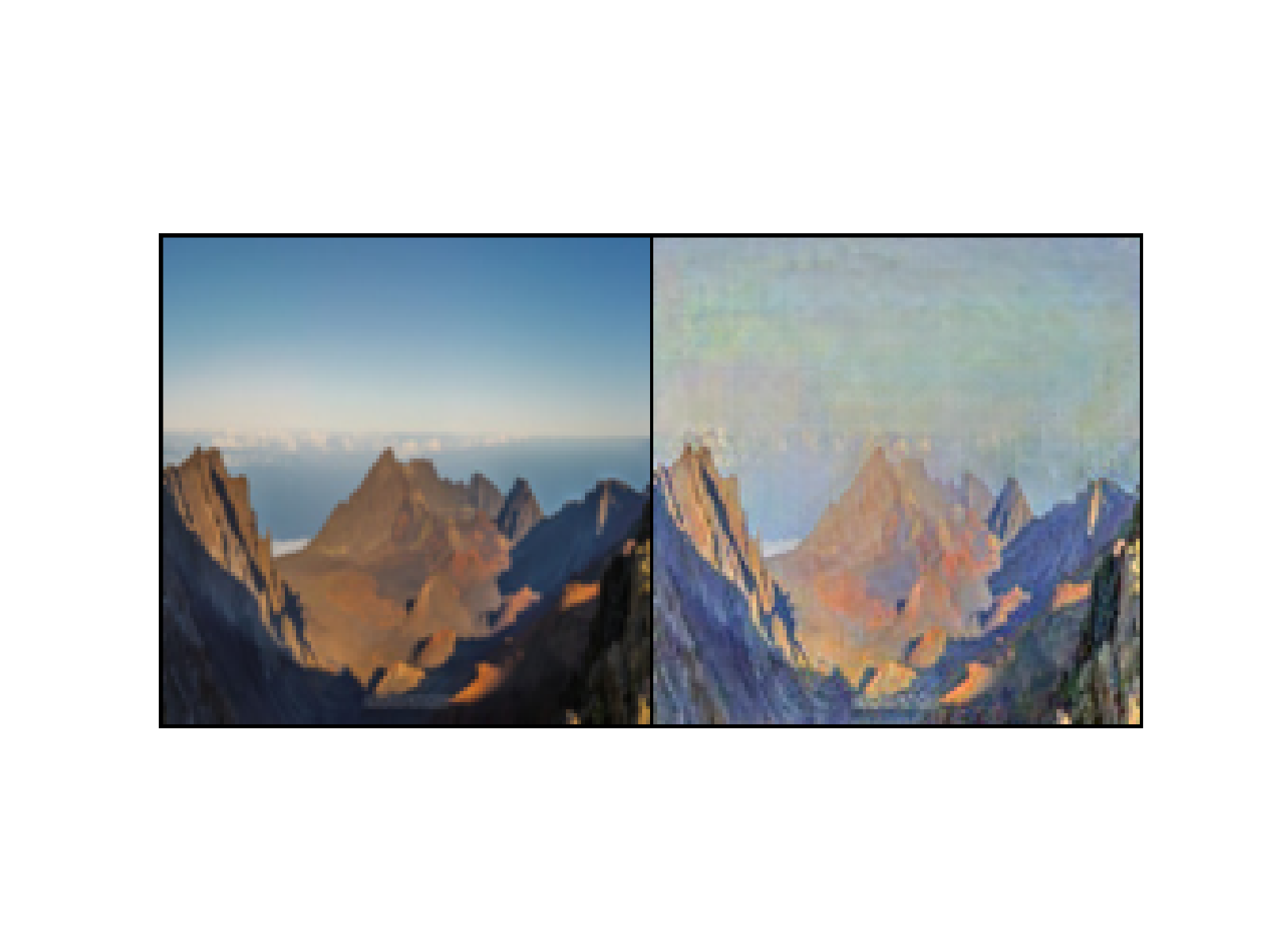}
    \vspace{0.00 in}
    \includegraphics[width = 0.58 \textwidth, viewport=50 75 420 260, clip = True]{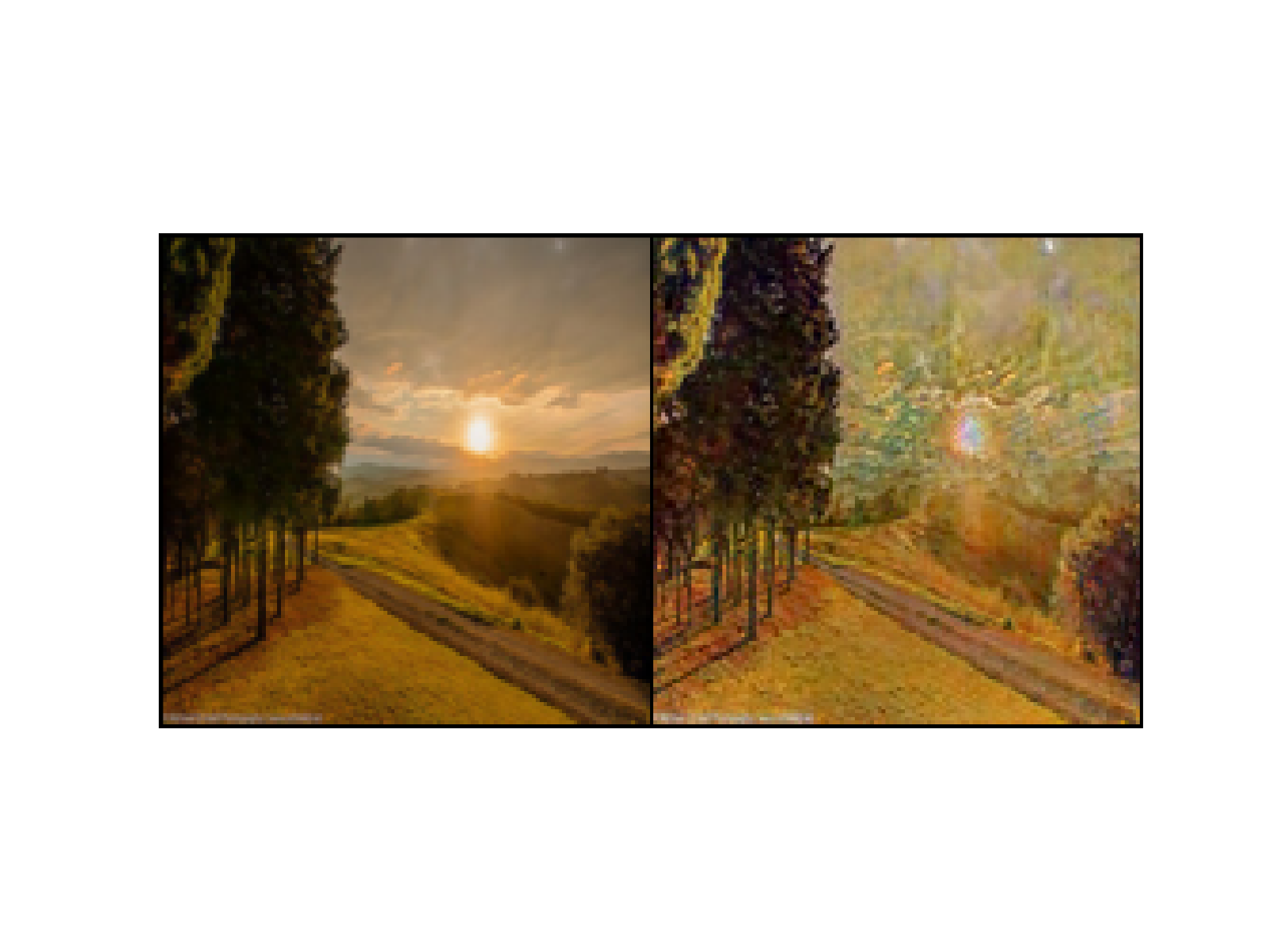}
    \includegraphics[width = 0.58 \textwidth, viewport=50 75 420 260, clip = True]{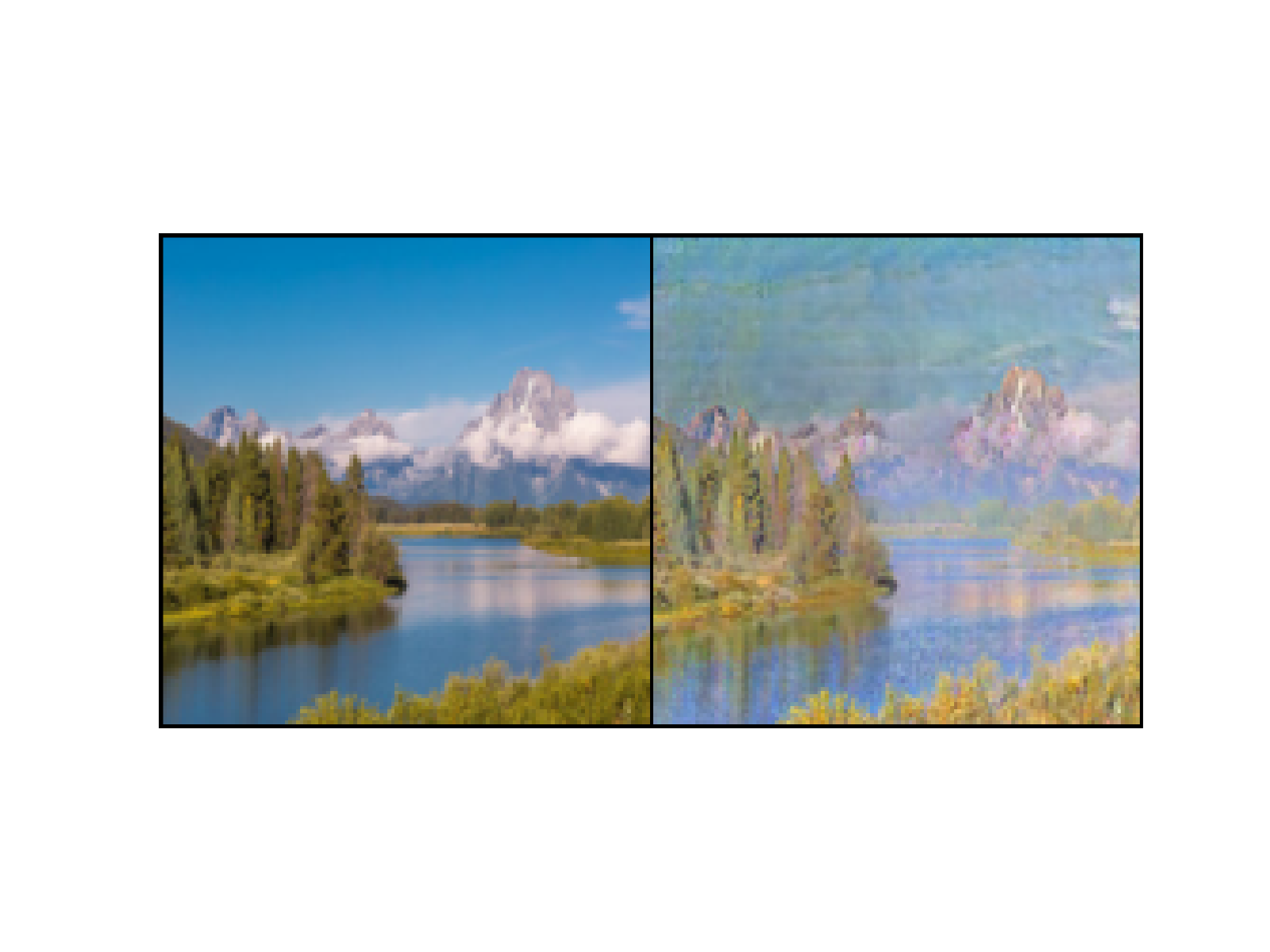}
    \includegraphics[width = 0.58 \textwidth, viewport=50 75 420 260, clip = True]{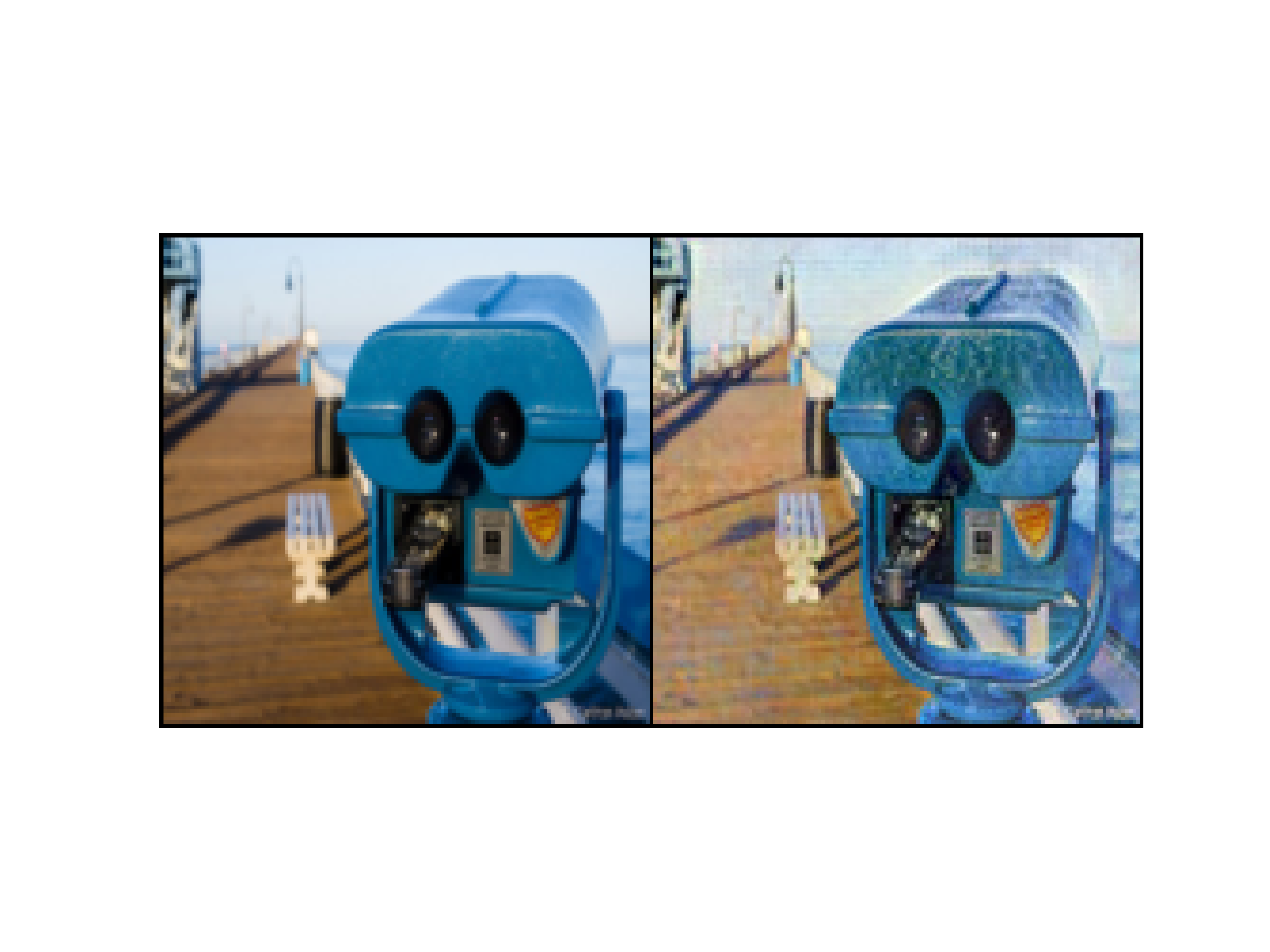}
    \caption{Additional examples of translating landscape photos into Monet paintings using TTC. As the fourth example shows, the quality of the translation worsens for source images which are far from Monet painting samples; since our step size is determined by the average transport distance, it is possible that it is not suitable for outliers with large initial transport distances.}
    \label{fig:more_monets}
\end{figure*}

\subsection{Additional image denoising examples}
\label{app:denoising}
Figure \ref{fig:more_denoising} includes additional noisy images from BSDS500, restored using the same sequence of critics and step sizes that were used to produce Figure \ref{fig:deer}. The corresponding PSNR values for each restored image are given in \Cref{table:moredenoisingresults}. Based on these PSNR values, the advantage of TTC over \cite{lunz2018adversarial} seems to increase when large parts of the image are roughly constant. Indeed, TTC has the largest improvement for the image of the cardinal, which has large constant portions. On the other hand, the improvement is smaller for the image of the grassy underbrush, which has many fine details. As a possible explanation for this trend, we note that on this dataset we have observed that the magnitude of the noise, summed over the entire image, is bigger than the estimated Wasserstein-1 distance, indicating that under the optimal transport some noisy images are not mapped to their original clean versions. For such images, one would expect that the improvement in PSNR after restoration would be lower.

\begin{figure*}
    \centering
    \includegraphics[width = 1.0 \textwidth, viewport=50 125 420 225, clip = True]{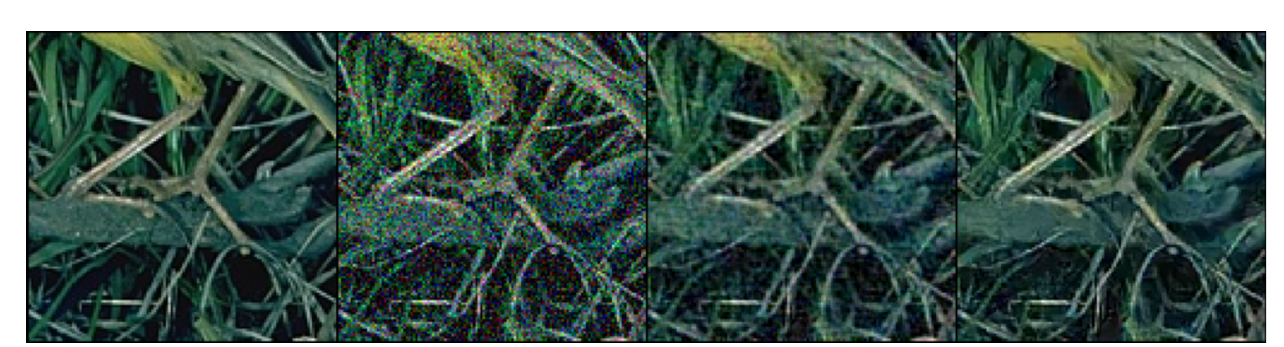}
    \includegraphics[width = 1.0 \textwidth, viewport=50 125 420 225, clip = True]{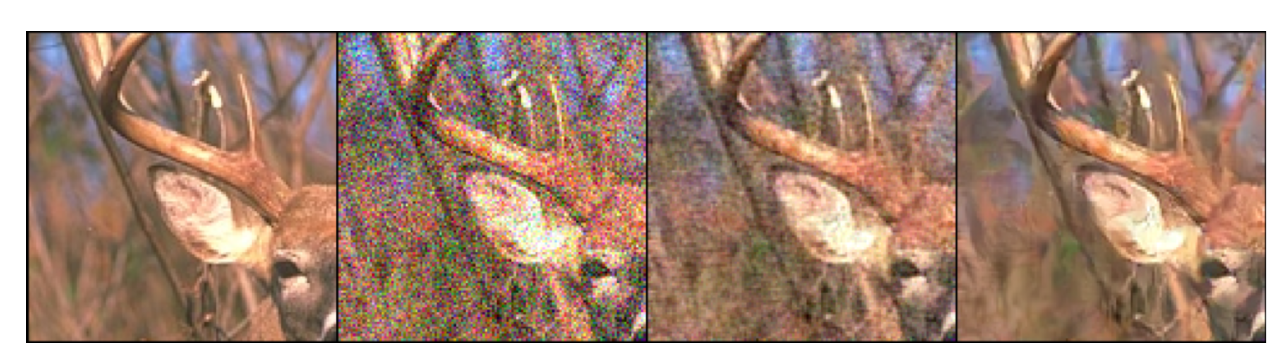}
    \includegraphics[width = 1.0 \textwidth, viewport=50 125 420 225, clip = True]{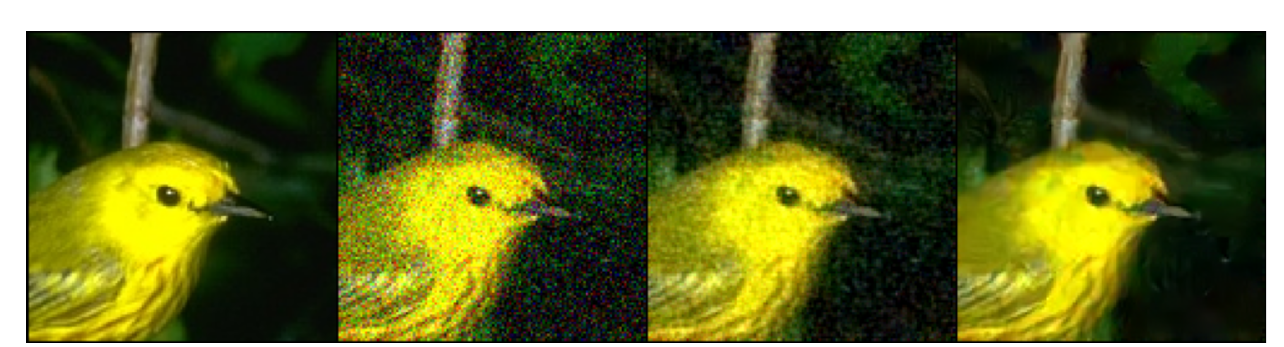}
    \includegraphics[width = 1.0 \textwidth, viewport=50 125 420 225, clip = True]{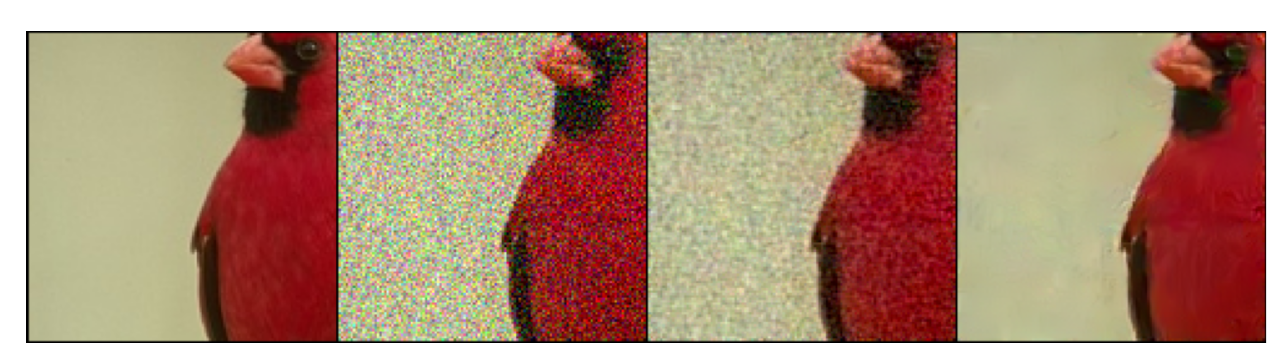}
    \caption{Restored images from BSDS500 with noise level $\sigma = 0.15$ using the same sequence of trained critics as described in Section \ref{sec:denoising}. From left to right: original image, noisy image, restored image using \cite{lunz2018adversarial}, restored image using TTC. For PSNR values, see Table \ref{table:moredenoisingresults}. The image of the deer is the same as the one in \Cref{sec:denoising}}.
    \label{fig:more_denoising}
\end{figure*}

\begin{table}[h]
\centering
\begin{tabular}{cccc}
\toprule
\multicolumn{1}{l}{} &
\multicolumn{3}{c}{PSNR (dB)} \\ 
\cmidrule(lr){2-4}
 Example & Noisy Image & Adv. Reg. \cite{lunz2018adversarial} & TTC  \\
\midrule
Underbrush & $16.4$ & $23.6$ & $\mathbf{24.8}$ \\
\midrule
Deer & $16.4$ &  $24.5$ & $ \mathbf{27.0}$ \\
\midrule
Yellow Bird & $16.5$ &  $24.8$ & $ \mathbf{29.2}$ \\
\midrule
Cardinal & $16.5$ & $25.3$ & $\mathbf{33.0}$ \\
\bottomrule
\end{tabular}
\caption{PSNR values for the images in \Cref{fig:more_denoising}. As mentioned above, the advantage of TTC increases when large parts of the initial image are almost constant.}
\label{table:moredenoisingresults}
\end{table}

\end{appendices}

\end{document}
https://www.overleaf.com/project/6189648a36768087755f6789